\pdfoutput=1
\documentclass{article}

\PassOptionsToPackage{numbers,square}{natbib}
\usepackage[main,final]{neurips_2026}

\usepackage[utf8]{inputenc} 
\usepackage[T1]{fontenc}    
\usepackage{hyperref}       
\usepackage{url}            
\usepackage{booktabs}       
\usepackage{amsfonts}       
\usepackage{nicefrac}       
\usepackage{microtype}      
\usepackage[table]{xcolor}  
\hypersetup{colorlinks=true,citecolor=gray,linkcolor=black,urlcolor=black}

\definecolor{firstbest}{HTML}{C8E0F4}   
\definecolor{secondbest}{HTML}{E5EEF7}  
\newcommand{\first}[1]{\cellcolor{firstbest}\textbf{#1}}
\newcommand{\second}[1]{\cellcolor{secondbest}\underline{#1}}

\usepackage{graphicx}
\usepackage{amsmath}
\usepackage{amssymb}
\usepackage{mathtools}
\usepackage{amsthm}
\usepackage{algorithm}
\usepackage{algorithmic}
\usepackage{multirow}
\usepackage{adjustbox}
\usepackage{dsfont}
\usepackage{wrapfig}
\usepackage{diagbox}
\usepackage{float}
\usepackage{needspace}
\usepackage[capitalize,noabbrev]{cleveref}
\usepackage[font=small]{caption}

\newtheorem{proposition}{Proposition}
\newtheorem{corollary}[proposition]{Corollary}

\title{EDISCO: \underline{E}quivariant \underline{DIS}crete Diffusion for Euclidean \underline{C}ombinatorial \underline{O}ptimization}

\author{%
  Ruogu Chen\thanks{Code is available at \url{https://github.com/ValleyC/EDISCO}.}\\
  Department of Electrical and\\
  Computer Engineering\\
  University of Alberta\\
  Edmonton, AB, Canada T6G 1H9\\
  \texttt{ruogu@ualberta.ca}
  \And
  Jie Han\\
  Department of Electrical and\\
  Computer Engineering\\
  University of Alberta\\
  Edmonton, AB, Canada T6G 1H9\\
  \texttt{jhan8@ualberta.ca}
}

\newlength{\ediscoOriginalTextFloatSep}
\newlength{\ediscoOriginalInTextSep}
\newlength{\ediscoOriginalFloatSep}
\begin{document}
\raggedbottom
\captionsetup{skip=4pt}
\setlength{\parskip}{3pt}
\setlength{\textfloatsep}{6pt plus 2pt minus 2pt}
\setlength{\intextsep}{4pt plus 2pt minus 2pt}
\setlength{\floatsep}{8pt plus 2pt minus 2pt}

\maketitle

\begin{abstract}
Euclidean combinatorial optimization problems (ECOPs), such as the Traveling Salesman Problem (TSP) and Capacitated Vehicle Routing Problem (CVRP), possess inherent symmetries under the two-dimensional Euclidean group E(2), including rotations, reflections, and translations. Existing learning-based methods, including recent diffusion-based methods, rely on data augmentation or regularization to approximate E(2)-equivariance. This paper presents EDISCO, the first discrete diffusion model for ECOPs with exact E(2)-invariant generative distributions over node-index solutions. EDISCO combines an E(2)-equivariant network that predicts edge-selection probabilities with a categorical continuous-time Markov chain over discrete edge variables. Exact posterior sampling supports efficient multi-step inference. The equivariant network gives EDISCO a local geometric inductive bias: edge neighborhoods with the same relative geometry and combinatorial context are represented consistently regardless of absolute position or orientation, making learning more efficient and inference more robust than non-equivariant methods. EDISCO outperforms previous learning-based state-of-the-art solvers on synthetic TSPs from 100 to 10000 nodes and CVRPs from 50 to 2000 customers, while using only 33--50\% of the training instances. Trained only on uniform synthetic data, EDISCO also outperforms competing learning-based baselines under spatial distribution shift and CVRP constraint-tightness shift.
\end{abstract}

\section{Introduction}
\label{sec:intro}

Euclidean combinatorial optimization problems (ECOPs), such as the Traveling Salesman Problem (TSP) and Capacitated Vehicle Routing Problem (CVRP), are core benchmarks for combinatorial optimization (CO) over spatial data. Classical solvers like Concorde and the LKH heuristic have been refined for decades~\citep{applegate2006concorde,helsgaun2017extension}, while learning-based solvers have emerged as fast approximate alternatives with amortized inference and instance-level generalization~\citep{kool2018attention,joshi2022learning}. Diffusion-based methods now lead the non-autoregressive family, pushing optimality gaps below one percent on synthetic benchmarks and scaling ECOP solving from hundreds to ten thousand nodes~\citep{sun2023difusco,li2023t2t,zhao2024disco,li2024fast,wang2025strudico,li2025gensco,wang2026nexco,macoexpander}. The field is now moving beyond chasing the last fraction of a percent on uniform synthetic benchmarks. Practical scenarios require solvers to remain reliable when the spatial distribution, problem size, or constraints shift away from training~\citep{bi2022learning,luo2025rethinking,bi2026car}, and robustness to such shifts is now a central criterion for whether learned solvers transfer beyond synthetic benchmarks.

A key obstacle is that ECOP solvers must learn reusable Euclidean structures, such as local distances, neighborhoods, edge crossings, depot-relative geometry, and capacity-pressure patterns. If a model recognizes these structures only in a particular coordinate frame, it must relearn the same patterns across rotated, reflected, and translated copies. The two-dimensional Euclidean group E(2) captures exactly these irrelevant coordinate dimensions. Rotations, reflections, and translations change the coordinate representation but not the node-index solution~\citep{ouyang2024generalization,bronstein2021geometric}. A reliable ECOP model should therefore learn geometric patterns that are constant on E(2) orbits, so the same local pattern induces the same decision wherever and however it appears in the plane. Existing neural CO methods approximate this principle through augmentation or symmetry regularization~\citep{kim2022sym}. However, these mechanisms do not enforce exact equivariance or invariance by construction~\citep{esteves2018learning,nordenfors2023optimization,brehmer2024does}. A diffusion model defines a generative distribution over solutions rather than a single deterministic prediction. Therefore, for diffusion-based CO, E(2) invariance must hold at the level of the entire distribution rather than only at the predicted output. Achieving this requires every stage of the pipeline (the diffusion model, reverse sampler, and decoder) to preserve the same symmetry. Existing diffusion solvers improve denoising quality, inference speed, local search, and expansion, but leave E(2) symmetry underexplored and do not guarantee exact E(2)-invariance of the generated solution distribution.

This work proposes EDISCO, the first discrete diffusion model for ECOPs with an end-to-end E(2)-invariant generative distribution over node-index solutions. EDISCO introduces an E(2)-equivariant edge-score network coupled with a categorical continuous-time Markov chain (CTMC)~\citep{campbell2022continuous} over discrete edge variables, so the reverse diffusion process preserves geometric symmetry at the distribution level. Exact CTMC posterior sampling adapts multi-step solvers to discrete combinatorial variables, giving a controllable quality--speed trade-off without retraining. To preserve E(2)-invariance through the final decoding step, EDISCO uses an invariant edge-expansion decoder that reads only edge probabilities, pairwise distances, and combinatorial feasibility constraints. The resulting model represents local edge neighborhoods with the same relative geometry consistently across positions and orientations, which makes learning more efficient and inference more robust than non-equivariant methods.

The contributions are as follows.
\begin{enumerate}
\item \emph{Exact distribution-level E(2) invariance for generative ECOP solving.}
EDISCO realizes orbit-shared local geometric structure by construction: scalar edge and message channels are constant on E(2) orbits, so equivalent local edge neighborhoods are represented consistently under E(2) transformations. The resulting generative distribution over node-index solutions is exactly E(2)-invariant. Proposition~\ref{prop:quotient_dimension} formalizes this induced quotient-space reduction.

\item \emph{Efficient continuous-time categorical diffusion for edge-valued CO.}
EDISCO formulates edge selection as a categorical CTMC and uses exact posterior sampling to adapt multi-step solvers to discrete combinatorial variables. This provides flexible quality--speed trade-offs without retraining.

\item \emph{Robust symmetry-preserving ECOP solving under spatial and constraint shifts.}
The same invariance principle is preserved under sparse graph construction and capacity-conditioned CVRP features. As a result, EDISCO transfers from uniform synthetic training to spatial out-of-distribution (OOD) benchmarks and capacity-shifted CVRP while maintaining the same geometric symmetry guarantee.
\end{enumerate}

EDISCO achieves leading optimality-gap results among the considered learning-based solvers on TSPs from 100 to 10000 nodes and on CVRPs from 50 to 2000 customers, while using only 33\% to 50\% of the training instances compared to competing methods. With the model only trained on synthetic data, EDISCO outperforms non-equivariant baselines under spatial distribution shift~\citep{bi2022learning} and achieves the lowest mean optimality gap among the considered learning-based methods on the evaluated TSPLIB instances~\citep{reinelt1991tsplib}. On CVRPs, the same invariant-feature principle supports capacity-conditioned routing under the protocol of~\citet{luo2025rethinking}, where a single pre-trained model spans a range of capacity ratios without coordinate-frame dependence. Additional experiments on the Euclidean Steiner Tree Problem (ESTP) and the Maximum Independent Set (MIS) demonstrate EDISCO's generality.

\setlength{\parskip}{5.5pt}
\section{Related Work}
\label{sec:related}

\paragraph{Neural and diffusion-based ECOP solvers.}
Autoregressive solvers such as AM~\citep{kool2018attention}, POMO~\citep{kwon2020pomo}, and BQ-NCO~\citep{drakulic2023bq} build tours sequentially, while heatmap methods predict pairwise edge structure and recover tours through search or local construction~\citep{joshi2019efficient,fu2021generalize,xin2021neurolkh}. Discrete diffusion~\citep{austin2021structured} extends this to generative modeling of edge variables, refined by DIFUSCO~\citep{sun2023difusco}, T2T~\citep{li2023t2t}, Fast-T2T~\citep{li2024fast}, CADO~\citep{yoon2024cado}, and DISCO~\citep{zhao2024disco}. Expansion solvers add local construction on top of global prediction, including COExpander~\citep{macoexpander} and the diffusion-expansion hybrids StruDiCO~\citep{wang2025strudico}, GenSCO~\citep{li2025gensco}, and NEXCO~\citep{wang2026nexco}. Across these methods, geometric symmetry is approached through orbit-sampling rather than architectural orbit-sharing, providing no exact distribution-level invariance guarantee.

\paragraph{Equivariance and geometric structure.}
Equivariant architectures enforce equivariance by construction~\citep{cohen2016group,bronstein2021geometric,satorras2021n}, applying a single scalar feature map across every E(2)-equivalent input, a property referred to here as \emph{orbit-sharing}. This restricts invariant functions to the quotient space and can improve data efficiency. Augmentation and symmetry regularization approximate the same effect through \emph{orbit-sampling}, without enforcing exact equivariance by construction~\citep{brehmer2024does,esteves2018learning,nordenfors2023optimization}. In neural CO, Sym-NCO~\citep{kim2022sym} samples orbits via a symmetry regularizer, while geometric models exploit Euclidean structure for TSP generalization~\citep{ouyang2024generalization,song2025geometrically}. Symmetry-aware diffusion has also been developed for molecular generation with denoising equivariant under the three-dimensional Euclidean group E(3)~\citep{hoogeboom2022equivariant} and for crystal structure prediction with lattice-permutation and periodic-translation symmetries~\citep{lin2024equivariant}.

\paragraph{Continuous-time categorical diffusion and decoding.}
Continuous-time categorical diffusion admits closed-form transitions and flexible inference schedules~\citep{campbell2022continuous,sun2022score}. Diffusion-based CO solvers recover discrete solutions through greedy extraction, local search, or expansion, with NEXCO~\citep{wang2026nexco} making expansion native to diffusion.

\paragraph{Constraint-aware and realistic routing.}
Neural solvers are increasingly evaluated beyond uniform synthetic TSP, including TSPLIB transfer~\citep{reinelt1991tsplib}, cross-distribution routing~\citep{bi2022learning}, real-world hierarchical TSP~\citep{goh2024hierarchical}, and generalist CO agents~\citep{drakulic2024goal}. For vehicle routing, RouteFinder~\citep{berto2024routefinder}, CaDA~\citep{li2024cada}, \citet{luo2025rethinking}, and CaR~\citep{bi2026car} target constraint-aware or capacity-robust routing.

\section{Method}
\label{sec:method}

\subsection{Problem Formulation}
An ECOP instance has $n$ nodes with coordinates $\{\mathbf{c}_i\}_{i=1}^n \subset \mathbb{R}^2$, and asks for a binary decision matrix $\mathbf{X} \in \{0,1\}^{n \times n}$ that minimizes a distance-based cost $f$ subject to problem-specific constraints $\mathcal{C}$:
{\small
\begin{equation}
    \mathbf{X}^* = \arg\min_{\mathbf{X}} f(\mathbf{X}, \{\mathbf{c}_i\}_{i=1}^n) \quad \text{s.t. } \mathbf{X} \in \mathcal{C}.
\end{equation}
}%
For TSPs, $\mathbf{X}$ is the tour adjacency, $f$ is the total edge length, and $\mathcal{C}$ encodes Hamiltonicity. EDISCO learns the conditional distribution $p(\mathbf{X}\mid\{\mathbf{c}_i\}_{i=1}^n)$ through a generative process~\citep{sun2023difusco,li2023t2t}.

\subsection{Categorical CTMC Diffusion over Edge Variables}

EDISCO models edge selection as a categorical diffusion process directly over the binary variables in $\mathbf{X}$, avoiding continuous relaxations of discrete tour structure and exposing closed-form transition probabilities for training and accelerated sampling~\citep{austin2021structured,campbell2022continuous}.

\paragraph{Forward process}
The CTMC corrupts the clean solution $\mathbf{X}_0$ with instantaneous transition rate
{\small
\begin{equation}
    \mathbf{Q}(t) = \beta(t) \left(\mathds{1}\mathds{1}^T - K\mathbf{I}\right),\label{eq:rate_matrix}
\end{equation}
}%
where $K=2$ for binary edges, $\mathds{1}\in\mathbb{R}^K$ is the all-ones vector, and $\mathbf{I}\in\mathbb{R}^{K\times K}$ is the identity matrix, with linear schedule $\beta(t) = \beta_{\min} + t(\beta_{\max} - \beta_{\min})$, $t \in [0,1]$, $\beta_{\min}=0.1$, $\beta_{\max}=1.5$ (alternatives are evaluated in Appendix~\ref{subsec:noise-schedule-experiment}). The Kolmogorov forward equation~\citep{norris1998markov} admits the closed-form transition probability~\citep{campbell2022continuous}
{\small
\begin{equation}
    P_{ab}(t|s) = \frac{1}{K} + \left(\delta_{ab} - \frac{1}{K}\right) \exp\left(-K \int_s^t \beta(u) du\right),
    \label{eq:transition_prob}
\end{equation}
}%
where $a, b \in \{0, \ldots, K-1\}$ are the indices of the categorical states. Specializing to binary edge variables ($K=2$, $a, b \in \{0, 1\}$),
{\small
\begin{equation}
    P(X_{t,ij} = b \mid X_{0,ij} = a) = P_{ab}(t|0) = \frac{1}{2} + \left(\delta_{ab} - \frac{1}{2}\right) e^{-2 \int_0^t \beta(u) du},
    \label{eq:forward_sampling}
\end{equation}
}%
with the linear schedule giving the analytic integral
{\small
\begin{equation}
    \int_0^t \beta(u) du = \beta_{\min}t + \frac{1}{2}(\beta_{\max} - \beta_{\min})t^2,
    \label{eq:beta_integral}
\end{equation}
}%
so $\mathbf{X}_t$ can be sampled exactly from $\mathbf{X}_0$ at any time without simulating intermediate jumps.

\paragraph{Learned reverse process}
Given $\mathbf{X}_0$, the exact one-step posterior follows from Bayes' rule,
{\small
\begin{equation}
    q(\mathbf{X}_{t-\Delta t}|\mathbf{X}_t, \mathbf{X}_0) = \frac{q(\mathbf{X}_t|\mathbf{X}_{t-\Delta t})\, q(\mathbf{X}_{t-\Delta t}|\mathbf{X}_0)}{q(\mathbf{X}_t|\mathbf{X}_0)}.
    \label{eq:posterior_bayes}
\end{equation}
}%
However, $\mathbf{X}_0$ is unknown at inference. The score network instead predicts $\hat{p}_\theta(\mathbf{X}_0|\mathbf{X}_t,t,\{\mathbf{c}_i\})$, and the learned reverse kernel marginalizes Equation~\ref{eq:posterior_bayes} over this predicted clean-edge distribution. This $x_0$ parameterization remains stable near low-noise timesteps where reconstruction accuracy matters most~\citep{salimans2022progressive}.

\subsection{E(2)-Equivariant Score Network}

Rotations, reflections, and translations of an ECOP coordinate frame leave the node-index solution unchanged, so EDISCO predicts edge variables with an E(n)-equivariant graph neural network (EGNN)~\citep{satorras2021n} that propagates coordinate transformations equivariantly while keeping edge logits invariant. Figure~\ref{fig:egnn-architecture} illustrates EDISCO's equivariant score-network architecture. The score network predicts clean-edge probabilities conditioned on the noisy edge state, diffusion time, and instance coordinates.

\begin{figure*}[t]
    \centering
    \includegraphics[width=\textwidth]{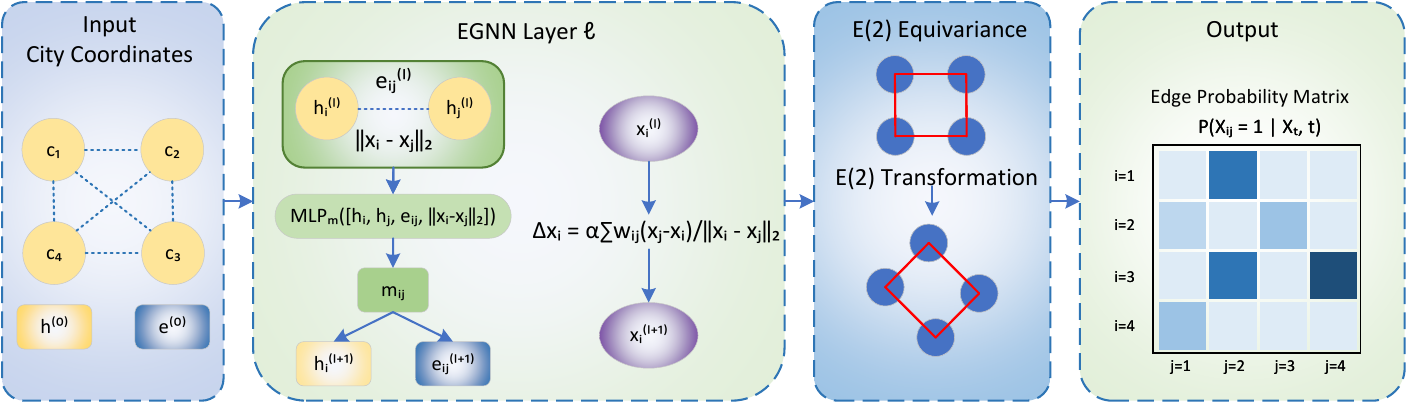}
    \caption{EDISCO's E(n)-equivariant graph neural network (EGNN) architecture. EGNN layers update coordinates equivariantly and output edge probabilities that remain invariant under rotations, reflections, and translations.}
    \label{fig:egnn-architecture}
\end{figure*}

The network maintains invariant node features $\mathbf{h}_i$, invariant edge features $\mathbf{e}_{ij}$, and equivariant coordinates $\mathbf{x}_i$. Pairwise messages, computed by a multilayer perceptron (MLP), use only invariant scalar geometry:
{\small
\begin{equation}
    \mathbf{m}_{ij}^{(\ell)} = \text{MLP}_m\left([\mathbf{h}_i^{(\ell)}, \mathbf{h}_j^{(\ell)}, \mathbf{e}_{ij}^{(\ell)}, \|\mathbf{x}_i^{(\ell)} - \mathbf{x}_j^{(\ell)}\|_2]\right).
    \label{eq:message}
\end{equation}
}%
Coordinate updates are constrained to scalar-weighted relative vectors,
{\small
\begin{equation}
    \Delta\mathbf{x}_i = \alpha \sum_{j \neq i} w_{ij} \cdot \frac{\mathbf{x}_j^{(\ell)} - \mathbf{x}_i^{(\ell)}}{\|\mathbf{x}_j^{(\ell)} - \mathbf{x}_i^{(\ell)}\|_2 + \epsilon},
    \label{eq:coord_update}
\end{equation}
}%
with $w_{ij} = \tanh(\text{MLP}_c(\mathbf{m}_{ij}^{(\ell)})/\tau)$, $\tau=10$, $\alpha=0.1$. Because $w_{ij}$ is built from invariant messages, the update transforms covariantly with the coordinates. Edge and node states then update as
{\small
\begin{equation}
    \mathbf{e}_{ij}^{(\ell+1)} = \text{LN}\big((\mathbf{e}_{ij}^{(\ell)} + \text{MLP}_e([\mathbf{e}_{ij}^{(\ell)}, \mathbf{m}_{ij}^{(\ell)}])) \odot (1 + \text{MLP}_t(\mathbf{t}_{\text{emb}}))\big),
    \label{eq:edge_update}
\end{equation}
}%
{\small
\begin{equation}
    \mathbf{h}_i^{(\ell+1)} = \text{LN}\big(\mathbf{h}_i^{(\ell)} + \text{MLP}_h([\mathbf{h}_i^{(\ell)}, \textstyle\sum_{j \neq i} \mathbf{m}_{ij}^{(\ell)}])\big),
    \label{eq:node_update}
\end{equation}
}%
where $\mathrm{LN}$ denotes layer normalization and $\mathbf{t}_{\text{emb}}$ is a sinusoidal time embedding~\citep{vaswani2017attention}. Because the output head reads only invariant edge states, the predicted distribution over node-index edge variables is invariant under E(2), and Corollary~\ref{cor:orbit_sharing} formalizes the orbit-shared scalar representation that follows. For large instances, EDISCO restricts message passing to an invariant $k$-nearest-neighbor ($k$NN) graph constructed from pairwise distances, with deterministic tie-breaking by node indices, so sparsification preserves the invariance argument.

\begin{proposition}
\label{prop:quotient_dimension}
Let $X = \mathbb{R}^{2n}$ denote the space of ordered two-dimensional coordinates for $n$ nodes, and let $G = \mathrm{E}(2)$ be the Euclidean group acting on $X$ by simultaneous rotation, reflection, and translation of all node positions. Let $X_{\mathrm{gp}} \subset X$ be the subset of configurations in general position (pairwise distinct points with no three collinear), and assume $n \ge 3$. Then:
\begin{enumerate}
    \item[(i)] The quotient space $X_{\mathrm{gp}}/G$ is a smooth manifold of dimension $2n - 3$.
    \item[(ii)] Any $G$-invariant function $F: X_{\mathrm{gp}} \to Y$ factors uniquely through the quotient as $F = \widetilde{F} \circ \pi$, where $\pi: X_{\mathrm{gp}} \to X_{\mathrm{gp}}/G$ is the canonical projection and $\widetilde{F}: X_{\mathrm{gp}}/G \to Y$ is a function on the quotient manifold.
    \item[(iii)] Learning a $G$-invariant function is equivalent to learning a function on the $(2n-3)$-dimensional manifold $X_{\mathrm{gp}}/G$ rather than on $\mathbb{R}^{2n}$.
\end{enumerate}
\end{proposition}

\begin{corollary}[Orbit-shared scalar representations]
\label{cor:orbit_sharing}
Let $\mathbf{x} = (\mathbf{c}_1, \ldots, \mathbf{c}_n)$ denote the input configuration and let $g \in \mathrm{E}(2)$ act as $g \cdot \mathbf{x} = (g \cdot \mathbf{c}_1, \ldots, g \cdot \mathbf{c}_n)$. Then at every EGNN layer $\ell$, the scalar features satisfy
\[
\mathbf{m}_{ij}^{(\ell)}(g \cdot \mathbf{x}) = \mathbf{m}_{ij}^{(\ell)}(\mathbf{x}), \quad \mathbf{e}_{ij}^{(\ell)}(g \cdot \mathbf{x}) = \mathbf{e}_{ij}^{(\ell)}(\mathbf{x}), \quad \mathbf{h}_i^{(\ell)}(g \cdot \mathbf{x}) = \mathbf{h}_i^{(\ell)}(\mathbf{x}),
\]
while the coordinate features transform equivariantly: $\mathbf{x}_i^{(\ell)}(g \cdot \mathbf{x}) = g \cdot \mathbf{x}_i^{(\ell)}(\mathbf{x})$. Consequently, any input configuration $\mathbf{x}$ and every element of its E(2)-orbit produce the same scalar feature field at every layer, so a single learned scalar feature map applies uniformly across the orbit.
\end{corollary}

The quotient dimension $2n-3$ in Proposition~\ref{prop:quotient_dimension} reflects E(2)'s three continuous degrees of freedom, namely two translations and one rotation, with reflections contributing only a discrete subgroup. The full proof and the derivation that this invariance is preserved through diffusion and decoding are in Appendices~\ref{subsec:proposition} and~\ref{subsec:equivarince}.

\paragraph{Capacity conditioning for CVRPs.}
\label{par:conditional_diffusion}
For CVRPs, capacity information enters the score network only through E(2)-invariant scalar channels. With vehicle capacity $Q$, default capacity $Q_{\mathrm{default}}$, ratio $\lambda = Q/Q_{\mathrm{default}}$, and demands $d_i$, each node receives $d_i/Q$, each edge receives $(d_i+d_j)/Q$ and $|d_i-d_j|/Q$, and the global input $[\log \lambda, \sum_i d_i/Q]$ is embedded as $\mathbf{z}_c$. This embedding modulates only scalar message channels via FiLM~\citep{perez2018film},
{\small
\begin{equation}
    \tilde{\mathbf{m}}_{ij}^{(\ell)} = \boldsymbol{\gamma}_\ell(\mathbf{z}_c) \odot \mathbf{m}_{ij}^{(\ell)} + \boldsymbol{\beta}_\ell(\mathbf{z}_c),
    \label{eq:film_modulation}
\end{equation}
}%
with $\boldsymbol{\gamma}_\ell$ and $\boldsymbol{\beta}_\ell$ produced by lightweight MLPs initialized near identity. Because every conditioning input is invariant and the coordinate update of Equation~\ref{eq:coord_update} is unchanged, the conditional distribution remains E(2)-invariant for every fixed capacity. Section~\ref{subsec:cvrp_constraint_shift} evaluates this variant under constraint shift.

Training minimizes cross-entropy on $x_0$ prediction with signal-to-noise ratio (SNR)-motivated weighting (Algorithm~\ref{alg:training}). Here, $s_\theta$ is the score network with parameter $\theta$, $\mathcal{L}$ is the weighted loss, and $\mathrm{CE}$ in Algorithm~\ref{alg:training} denotes edge-wise cross-entropy. AdamW updates $\theta$ using $\nabla_\theta\mathcal{L}$, the loss gradient. For CVRPs, each sampled instance supplies demands $d_i$ and capacity $Q$ to the decoder. Capacity conditioning $\mathbf{z}_c$ and these inputs are omitted for TSPs. At each step a random time $t$ and clean instances are sampled, the noisy state is drawn from the closed-form forward process, and the $(1-\sqrt{t})$ weighting emphasizes low-noise timesteps where reconstruction targets are reliable~\citep{kingma2021variational,hang2023efficient}. Full training configuration is in Appendix~\ref{subsec:training_config}.

\subsection{Multi-Step CTMC Sampling}
\label{subsec:multi_step_sampling}

The continuous-time formulation allows the inference schedule to be changed without retraining. EDISCO adapts multi-step solvers such as PNDM~\citep{liu2022pseudo} and DEIS~\citep{zhang2022fast} to categorical CTMCs by separating prediction smoothing from discrete posterior sampling. At each reverse step, the score network predicts $\hat{p}_\theta(\mathbf{X}_0\mid\mathbf{X}_t,t,\{\mathbf{c}_i\})$, and Adams-Bashforth (AB)-style smoothing combines recent predictions into $\tilde{p}$. Smoothing starts with the current prediction alone, uses a lower order until its history fills, and clamps probabilities to $[0,1]$. For each edge $(i,j)$, sampling then uses the explicit categorical kernel
{\small
\begin{equation}
    q(x_{s,ij} \mid x_{t,ij}, \tilde{p}_{ij}) = \sum_{x_0 \in \{0,1\}} \frac{P(x_{t,ij} \mid x_{s,ij})\, P(x_{s,ij} \mid x_0)}{P(x_{t,ij} \mid x_0)}\, \tilde{p}_{ij}(x_0 \mid x_{t,ij}),
    \label{eq:exact_posterior}
\end{equation}
}%
where $s<t$, $\tilde{p}_{ij}(x_0\mid x_{t,ij})$ is the smoothed probability of clean state $x_0$, and the transition probabilities $P(\cdot\mid\cdot)$ follow from Equation~\ref{eq:transition_prob}. Edges are sampled independently from this kernel to form $\mathbf{X}_s$. This preserves the categorical transition structure while giving a controllable quality--speed trade-off. Appendix~\ref{subsec:ctmc_solvers} details prediction smoothing and posterior marginalization, and Appendix~\ref{subsec:solver_eval} evaluates solver choices.

\subsection{Native Edge Expansion (NEE)}

Denoising outputs an edge-probability matrix $P \in [0,1]^{n \times n}$ with $P_{ij}=p(X_{ij}=1)$. NEE converts this invariant score field into a feasible tour without external local-search refinement. For TSPs, undirected candidates are ranked by
{\small
\begin{equation}
    r_{ij} = \frac{P_{ij}+P_{ji}}{2(d_{ij}+\epsilon)},
    \qquad d_{ij}=\|\mathbf{c}_i-\mathbf{c}_j\|_2,
    \label{eq:nee_score}
\end{equation}
}%
and the highest-ranked feasible edge is added iteratively. The feasibility projection maintains node degrees, a partial edge set $S_t$, and a union-find structure, accepting an edge only when both endpoints have degrees below two and the edge does not close a premature subtour, with the final Hamiltonian closure as the only allowed cycle-closing step. For CVRPs, customer nodes keep degree two while the depot accepts one edge pair per active route. The projection tracks per-route accumulated demand, and a candidate edge that would push any active route's demand above $Q$ is rejected. The matching depot-return edge closes that partial route before the next candidate is evaluated. All inputs used by the decoder (pairwise distances, edge probabilities, capacity-derived features, partial edge sets, union-find connectivity) and all feasibility checks are invariant under E(2), so the selected node-index edge set is identical under any $g \in \mathrm{E}(2)$, with the full preservation argument in Appendix~\ref{subsec:equivarince}.

Algorithm~\ref{alg:inference} initializes $\mathbf{X}_{t_0}$ from the CTMC equilibrium, with the residual mismatch absorbed within the first denoising steps. The schedule $\{t_k\}_{k=0}^{T}$ contains $T$ reverse steps from $t_0=1$ to $t_T=0$, followed by a final clean-state prediction at $t_T=0$. Greedy extraction and 2-opt local search~\citep{croes1958method} remain reference baselines, but NEE is the native score-to-tour decoder whose invariance follows from its inputs and feasibility projection.

\begin{figure}[!hbp]
\centering
\begin{minipage}[t]{0.48\textwidth}
\vspace{0pt}
\begin{algorithm}[H]
\caption{EDISCO Training}
\label{alg:training}
\small
\begin{algorithmic}[1]
\REQUIRE Dataset $\mathcal{D}$, score network $s_\theta$
\REQUIRE CVRP only: $d_i$, $Q$, $Q_{\mathrm{default}}$
\WHILE{not converged}
  \STATE $(\mathbf{X}_0, \{\mathbf{c}_i\}) \sim \mathcal{D}$, \,\, $t \sim \mathcal{U}(0,1)$
  \STATE $\mathbf{z}_c \gets \mathrm{MLP}([\log \lambda, \textstyle\sum_i d_i / Q])$ if CVRP
  \STATE $\mathbf{X}_t \sim q(\mathbf{X}_t | \mathbf{X}_0)$ (Eq.~\ref{eq:forward_sampling})
  \STATE $\hat{\mathbf{X}}_0 \gets s_\theta(\mathbf{X}_t, t, \{\mathbf{c}_i\}, \mathbf{z}_c)$
  \STATE $\mathcal{L} \gets (1 - \sqrt{t}) \cdot \mathrm{CE}(\hat{\mathbf{X}}_0, \mathbf{X}_0)$
  \STATE $\theta \gets \mathrm{AdamW}(\theta, \nabla_\theta \mathcal{L})$
\ENDWHILE
\end{algorithmic}
\end{algorithm}
\end{minipage}
\hfill
\begin{minipage}[t]{0.48\textwidth}
\vspace{0pt}
\begin{algorithm}[H]
\caption{EDISCO Inference}
\label{alg:inference}
\small
\begin{algorithmic}[1]
\REQUIRE $\{\mathbf{c}_i\}$, trained score network $s_\theta$
\REQUIRE Schedule $\{t_k\}_{k=0}^{T}$
\REQUIRE CVRP only: $d_i$, $Q$, $Q_{\mathrm{default}}$
\STATE $\mathbf{X}_{t_0} \sim \mathrm{Uniform}(\{0,1\}^{n \times n})$
\STATE $\mathbf{z}_c \gets \mathrm{MLP}([\log \lambda, \textstyle\sum_i d_i / Q])$ if CVRP
\FOR{$k = 0$ to $T-1$}
  \STATE $\hat{\mathbf{X}}_0 \gets s_\theta(\mathbf{X}_{t_k}, t_k, \{\mathbf{c}_i\}, \mathbf{z}_c)$
  \STATE $\tilde{p} \gets \mathrm{ABStyleSmooth}(\{\hat{\mathbf{X}}_0^{(t_j)}\}_{j \le k})$
  \STATE $\mathbf{X}_{t_{k+1}} \sim q(\cdot | \mathbf{X}_{t_k}, \tilde{p})$ (Eq.~\ref{eq:exact_posterior})
\ENDFOR
\STATE $\hat{\mathbf{X}}_0 \gets s_\theta(\mathbf{X}_{t_T}, t_T, \{\mathbf{c}_i\}, \mathbf{z}_c)$
\STATE \textbf{return} $\mathrm{NEE}(\hat{\mathbf{X}}_0, \{\mathbf{c}_i\}, \{d_i\}, Q)$
\end{algorithmic}
\end{algorithm}
\end{minipage}
\end{figure}

\section{Experiments}
\label{sec:experiments}

We use $N$ for the number of TSP cities or CVRP customers. Main results cover TSP-100/500/1000 and CVRP-50/100/200/500, with TSP-10000 and CVRP-1000/2000 extended in Appendices~\ref{subsec:tsp10000} and~\ref{app:partition_diffusion}. ESTP (Appendix~\ref{app:steiner}) and MIS (Appendix~\ref{app:mis}) test Euclidean tree construction and engine generality on a non-Euclidean problem.

\subsection{Setup}

\paragraph{Datasets.} TSPs follow~\citet{kool2018attention} with coordinates uniform on $[0,1]^2$, Concorde~\citep{applegate2006concorde} labels at TSP-100, LKH-3~\citep{helsgaun2017extension} labels at TSP-500/1000/10000, and the standard 1{,}280-instance and 128-instance test sets from~\citet{kool2018attention} and~\citet{fu2021generalize}. EDISCO uses only 33--50\% of the training data of competing diffusion methods at every TSP and CVRP scale (Appendices~\ref{subsec:model_efficiency} and~\ref{subsec:training_config}). CVRP sizes, capacities, and test-set counts follow~\citet{macoexpander} (Appendix H.2). We use integer demands $\{1,\ldots,9\}$, capacities $Q\in\{40,50,80,100\}$ for $N\in\{50,100,200,500\}$, HGS-CVRP labels~\citep{vidal2022hybrid}, and 10K/10K/100/100 test instances at $N\in\{50,100,200,500\}$, with gaps reported against HGS-CVRP since exact optima are unavailable. The capacity-shift evaluation (Section~\ref{subsec:cvrp_constraint_shift}) follows~\citet{luo2025rethinking}.

\paragraph{Graph representation.} Dense adjacency is used for $N \le 100$ and $k$NN sparsification~\citep{sun2023difusco} at larger scales. Because neighbor sets depend only on pairwise Euclidean distances, sparsification itself is E(2)-invariant, so the invariance argument carries over unchanged. Memory drops from $O(n^2)$ to $O(kn)$, enabling training and inference on a single 48\,GB GPU at all reported scales (Appendix~\ref{subsec:training_resources}).

\subsection{TSP Results}
\label{subsec:tsp_main}

\Needspace{18\baselineskip}
\begin{wrapfigure}{l}{0.3\textwidth}
\centering
\includegraphics[width=\linewidth]{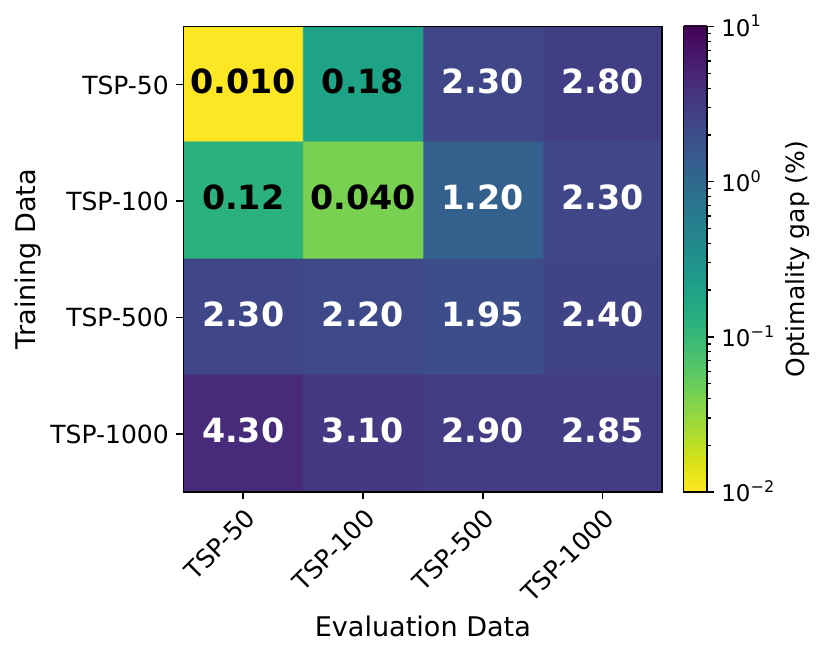}
\caption{Cross-size generalization for TSP models trained at one scale and evaluated across all scales under greedy decoding. Values are optimality gaps (\%).}
\label{fig:generalization}
\end{wrapfigure}

Table~\ref{tab:tsp_main} reports TSPs with $N \in \{100, 500, 1000\}$ with canonical configurations following~\citet{li2025gensco}, with each method's decoding configuration listed in the table. EDISCO uses 5-step DEIS-2~\citep{zhang2022fast} with NEE. TSP-10000, the 50-step PNDM~\citep{liu2022pseudo} quality variant, sampling-based decoding, and TSPLIB are in Appendices~\ref{subsec:tsp10000} and~\ref{subsec:TSPLIB}.

\begin{table*}[!p]
\centering
\caption{TSP results at $N \in \{100, 500, 1000\}$. Gap is relative to Concorde (TSP-100) or LKH-3 (TSP-500/1000). Time is per-instance wall-clock at batch size 1, single-thread~\citep{wang2025strudico,wang2026nexco}, with hardware in Appendix~\ref{subsec:hardware}. Throughout the tables, $\downarrow$ indicates lower is better and $\uparrow$ indicates higher is better. \first{Best} and \second{second-best} per Gap and Time among modern diffusion- and expansion-based solvers (COExpander and Generative Neural Solvers). RL: reinforcement learning, SL: supervised learning, G: greedy, BS: beam search, 2O: 2-opt, S: sampling, AE/NAE/NEE: adaptive / native adaptive / native edge expansion.}
\label{tab:tsp_main}
\scriptsize
\setlength{\tabcolsep}{3pt}
\adjustbox{max width=\textwidth}{%
\begin{tabular}{llccccccccc}
\toprule
\multirow{2}{*}{\textbf{Method}} & \multirow{2}{*}{\textbf{Type}} & \multicolumn{3}{c}{\textbf{TSP-100}} & \multicolumn{3}{c}{\textbf{TSP-500}} & \multicolumn{3}{c}{\textbf{TSP-1000}} \\
\cmidrule(lr){3-5} \cmidrule(lr){6-8} \cmidrule(lr){9-11}
& & Length$\downarrow$ & Gap$\downarrow$ & Time$\downarrow$ & Length$\downarrow$ & Gap$\downarrow$ & Time$\downarrow$ & Length$\downarrow$ & Gap$\downarrow$ & Time$\downarrow$ \\
\midrule
\multicolumn{11}{l}{\textit{Reference Solvers}} \\
Concorde~\citep{applegate2006concorde}    & Exact     & 7.76         & 0.00\%             & 0.23 s    & 16.55          & 0.00\%             & 18.65 s   & 23.12          & 0.00\%             & 84.41 s \\
LKH-3~\citep{helsgaun2017extension}       & Heuristic & 7.76         & 0.00\%             & 0.12 s    & 16.55          & 0.00\%             & 4.64 s    & 23.12          & 0.00\%             & 17.54 s \\
\midrule
\multicolumn{11}{l}{\textit{Non-Generative Neural Solvers}} \\
AM~\citep{kool2018attention}              & RL+BS      & 7.95         & 2.48\%             & 0.61 s    & 19.53          & 18.03\%            & 1.31 s    & 29.90          & 29.24\%            & 5.91 s  \\
POMO~\citep{kwon2020pomo}                 & RL+$\times$8 & 7.77       & 0.13\%             & 0.006 s   & 20.22          & 22.19\%            & 0.47 s    & 32.55          & 40.57\%            & 3.75 s  \\
GCN~\citep{joshi2019efficient}            & SL+BS      & 8.41         & 8.38\%             & 0.28 s    & 30.37          & 83.55\%            & 17.81 s   & 51.26          & 121.73\%           & 24.23 s \\
DIMES~\citep{qiu2022dimes}                & RL+G       & 8.01         & 3.23\%             & 0.06 s    & 17.17          & 3.74\%             & 0.45 s    & 24.79          & 7.22\%             & 1.12 s  \\
BQ-NCO~\citep{drakulic2023bq}             & SL+BS      & 7.76         & 0.01\%             & 0.19 s    & 16.64          & 0.55\%             & 7.03 s    & 23.47          & 1.38\%             & 57.81 s \\
GLOP~\citep{ye2024glop}                   & RL+G       & --           & --                 & --        & 17.10          & 3.32\%             & 0.17 s    & 23.84          & 3.11\%             & 1.40 s  \\
COExpander~\citep{macoexpander}           & SL+AE+2O   & 7.76         & 0.04\%             & 0.18 s    & 16.63          & 0.49\%             & 0.61 s    & 23.34          & 0.95\%             & 2.26 s  \\
\midrule
\multicolumn{11}{l}{\textit{Generative Neural Solvers}} \\
DIFUSCO~\citep{sun2023difusco}            & SL+G+2O    & 7.78         & 0.22\%             & 0.59 s    & 16.81          & 1.55\%             & 1.43 s    & 23.55          & 1.86\%             & 5.04 s  \\
T2T~\citep{li2023t2t}                     & SL+G+2O    & 7.76         & 0.06\%             & 1.13 s    & 16.68          & 0.78\%             & 3.05 s    & 23.41          & 1.25\%             & 9.23 s  \\
Fast-T2T~\citep{li2024fast}               & SL+G+2O    & 7.76         & 0.03\%             & 1.10 s    & 16.61          & 0.39\%             & 1.39 s    & 23.25          & 0.58\%             & 5.78 s  \\
CADO~\citep{yoon2024cado}                 & SL+RL+G+2O & 7.76         & \second{0.01\%}    & 0.75 s    & 16.59          & 0.24\%             & 2.50 s    & 23.28          & 0.69\%             & 8.50 s  \\
DISCO~\citep{zhao2024disco}               & SL+G       & 7.80         & 0.58\%             & 0.50 s    & --             & --                 & --        & --             & --                 & --      \\
StruDiCO~\citep{wang2025strudico}         & SL+G+2O    & 7.76         & 0.024\%            & 0.083 s   & 16.60          & 0.326\%            & 0.674 s   & 23.24          & \second{0.535\%}   & 2.614 s \\
GenSCO~\citep{li2025gensco}               & SL (C=10)  & 7.76         & \first{0.000\%}    & 0.50 s    & 16.55          & \second{0.023\%}   & 1.10 s    & 23.25          & 0.559\%            & 3.50 s  \\
NEXCO~\citep{wang2026nexco}               & SL+NAE     & 7.76         & 0.03\%             & \second{0.08 s} & 16.59    & 0.28\%             & \second{0.33 s} & 23.26    & 0.63\%             & \second{1.31 s} \\
\textbf{EDISCO (this work)}                    & SL+NEE     & 7.76         & \first{0.000\%}    & \first{0.075 s} & 16.55    & \first{0.018\%}    & \first{0.28 s} & 23.24    & \first{0.52\%}     & \first{1.28 s} \\
\bottomrule
\end{tabular}
}
\end{table*}

\paragraph{TSP across scales.} EDISCO leads on Gap at all three scales and is the fastest among modern diffusion- and expansion-based solvers. Faster non-generative baselines such as POMO exist but at much higher Gap (above $22\%$ at TSP-500/1000). EDISCO also leads both Gap and Time among the evaluated learning-based solvers with sub-$2\%$ Gap at TSP-10000 (Appendix~\ref{subsec:tsp10000}), with compute and runtime conventions in Appendix~\ref{subsec:recent_competitors_pareto}. The TSP-1000-trained model in Figure~\ref{fig:generalization} has gaps of at most $4.3\%$ across all other scales, exceeding the cross-size generalization of DIFUSCO~\citep{sun2023difusco} and T2T~\citep{li2023t2t}.

\par
\ifnum\value{WF@wrappedlines}>1
  \vspace{\dimexpr\value{WF@wrappedlines}\baselineskip-\baselineskip\relax}
\fi
\WFclear
\begin{table}[!htbp]
\centering
\caption{Cross-distribution generalization on TSP-100~\citep{bi2022learning}, EDISCO and diffusion baselines trained on Uniform with greedy + 2-opt decoding. EDISCO's SL+NEE configuration appears in Table~\ref{tab:tsp_main}. Det.(\%): deterioration relative to Uniform. Average Gap includes Uniform and the three OOD variants, while average Det.\ covers the three OOD variants. \first{Best} and \second{second-best} per column, including ties.}
\label{tab:ood_results}
\footnotesize
\setlength{\tabcolsep}{4pt}
\adjustbox{max width=\textwidth}{%
\begin{tabular}{lcccccccc}
\toprule
\multirow{2}{*}{\textbf{Method}} & \multicolumn{2}{c}{\textbf{Cluster}} & \multicolumn{2}{c}{\textbf{Explosion}} & \multicolumn{2}{c}{\textbf{Implosion}} & \multicolumn{2}{c}{\textbf{Average}} \\
\cmidrule(lr){2-3} \cmidrule(lr){4-5} \cmidrule(lr){6-7} \cmidrule(lr){8-9}
& Gap$\downarrow$ & Det.$\downarrow$ & Gap$\downarrow$ & Det.$\downarrow$ & Gap$\downarrow$ & Det.$\downarrow$ & Gap$\downarrow$ & Det.$\downarrow$ \\
\midrule
AM~\citep{kool2018attention} & 17.97\% & 678\% & 3.82\% & 65\% & 2.43\% & 5\% & 6.63\% & 249\% \\
AMDKD+EAS~\citep{bi2022learning} & \second{0.17\%} & 112\% & \second{0.05\%} & \second{-39\%} & \second{0.08\%} & \second{1\%} & \second{0.09\%} & 25\% \\
\midrule
DIFUSCO~\citep{sun2023difusco} & 2.87\% & 184\% & 1.38\% & 37\% & 2.80\% & 177\% & 2.02\% & 133\% \\
T2T~\citep{li2023t2t} & 1.50\% & 733\% & 0.15\% & -17\% & 2.60\% & 1344\% & 1.11\% & 687\% \\
Fast-T2T~\citep{li2024fast} & 1.18\% & 1867\% & \first{0.03\%} & \first{-52\%} & 2.50\% & 4067\% & 0.94\% & 1961\% \\
GLOP~\citep{ye2024glop} & \second{0.17\%} & \second{82\%} & 0.07\% & -28\% & \second{0.08\%} & \first{-10\%} & 0.10\% & \second{15\%} \\
\midrule
\textbf{EDISCO (this work)} & \first{0.05\%} & \first{25\%} & \first{0.03\%} & -25\% & \first{0.05\%} & 13\% & \first{0.04\%} & \first{4\%} \\
\bottomrule
\end{tabular}%
}
\end{table}

\Needspace{22\baselineskip}
\begin{wrapfigure}{r}{0.30\textwidth}
\centering
\includegraphics[width=\linewidth]{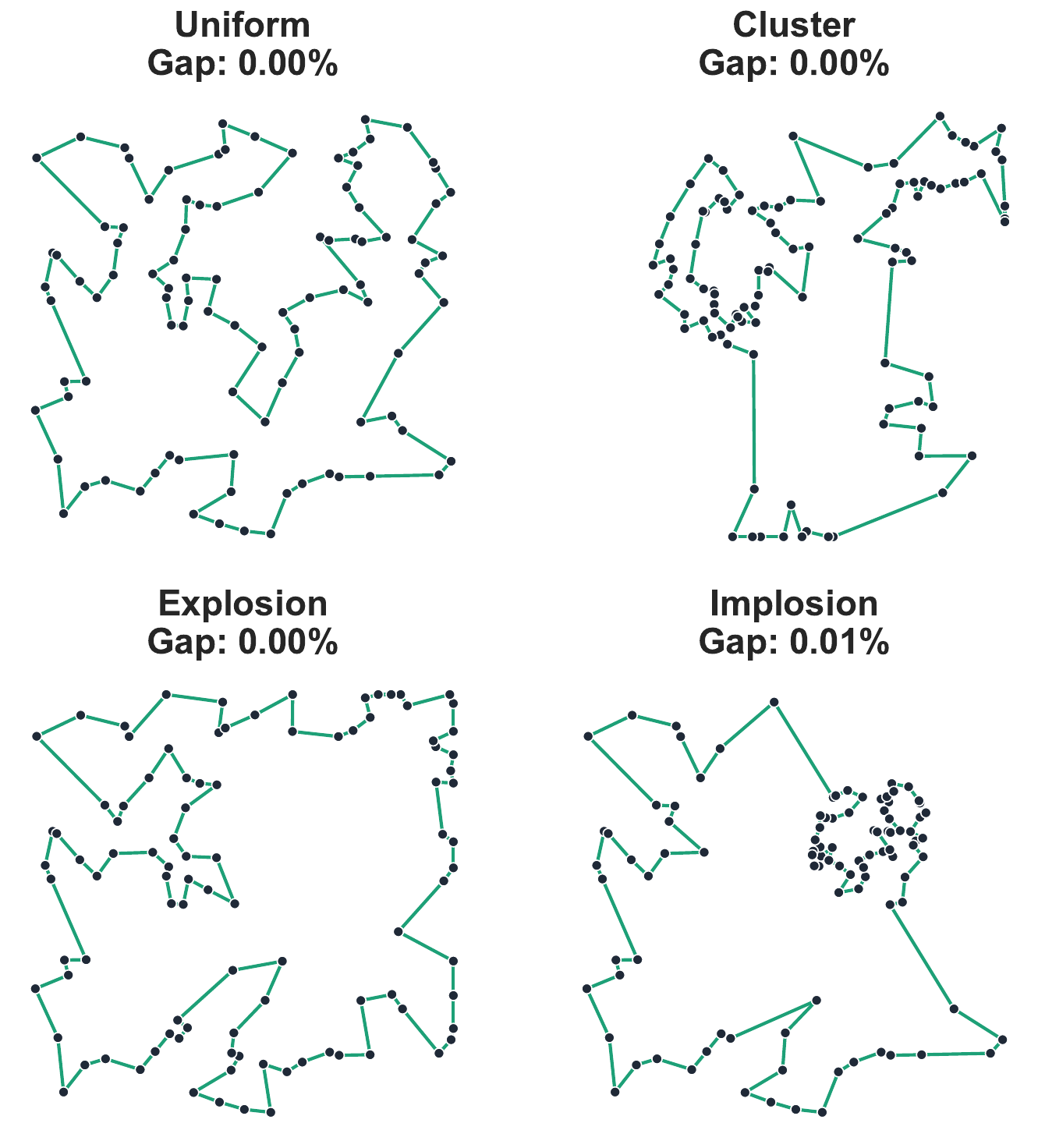}
\caption{EDISCO tours on Uniform and three OOD distributions for TSP-100.}
\label{fig:ood_visualization}
\end{wrapfigure}

\paragraph{Cross-distribution generalization.} Following~\citet{bi2022learning}, we evaluate on Uniform together with three OOD variants Cluster, Explosion, and Implosion (Table~\ref{tab:ood_results}, $\text{Det.}(\%) = (\text{Gap}_{\text{OOD}}/\text{Gap}_{\text{Uniform}} - 1)\times 100$). EDISCO's average deterioration is $4\%$ against $15\%$ (GLOP), $133\%$ (DIFUSCO), $687\%$ (T2T), and $1961\%$ (Fast-T2T), consistent with Corollary~\ref{cor:orbit_sharing}: orbit-shared scalar features depend only on local geometry rather than on coordinate-frame patterns of the training distribution. Tour visualizations are in Figure~\ref{fig:ood_visualization} and Appendix~\ref{subsec:visualization}.

\subsection{CVRP Results}
\label{subsec:cvrp_main}

Table~\ref{tab:cvrp_main} reports CVRPs with $N \in \{50, 100, 200, 500\}$ against HGS-CVRP. EDISCO extends the NEE decoder in Table~\ref{tab:tsp_main} with capacity-feasibility checks and capacity-conditioned invariant features (Section~\ref{par:conditional_diffusion}), preserving E(2)-invariance through the constrained-routing pipeline.

\WFclear
\begin{table*}[!htbp]
\centering
\caption{CVRPs at $N \in \{50,100,200,500\}$. Gap is relative to HGS-CVRP~\citep{vidal2022hybrid}. Time follows the protocol of Table~\ref{tab:tsp_main}. StruDiCO is adapted to CVRP from its TSP/MIS formulation. \first{Best} and \second{second-best} per Gap and Time among neural solvers. G: greedy, AE/NAE/NEE: adaptive / native adaptive / native edge expansion, LS: classical local-search post-processing. ``--'': not reported.}
\label{tab:cvrp_main}
\scriptsize
\setlength{\tabcolsep}{3pt}
\adjustbox{max width=\textwidth}{%
\begin{tabular}{llcccccccccccc}
\toprule
\multirow{2}{*}{\textbf{Method}} & \multirow{2}{*}{\textbf{Type}} & \multicolumn{3}{c}{\textbf{CVRP-50}} & \multicolumn{3}{c}{\textbf{CVRP-100}} & \multicolumn{3}{c}{\textbf{CVRP-200}} & \multicolumn{3}{c}{\textbf{CVRP-500}} \\
\cmidrule(lr){3-5} \cmidrule(lr){6-8} \cmidrule(lr){9-11} \cmidrule(lr){12-14}
& & Length$\downarrow$ & Gap$\downarrow$ & Time$\downarrow$ & Length$\downarrow$ & Gap$\downarrow$ & Time$\downarrow$ & Length$\downarrow$ & Gap$\downarrow$ & Time$\downarrow$ & Length$\downarrow$ & Gap$\downarrow$ & Time$\downarrow$ \\
\midrule
\multicolumn{14}{l}{\textit{Reference Solvers}} \\
HGS-CVRP~\citep{vidal2022hybrid}                       & Heuristic  & 10.37 & 0.00\%       & 1.01 s  & 15.56 & 0.00\%       & 20.03 s & 19.63 & 0.00\%       & 60.02 s & 37.15 & 0.00\%       & 240 s  \\
\midrule
\multicolumn{14}{l}{\textit{Non-Generative Neural Solvers}} \\
Sym-NCO~\citep{kim2022sym}                             & RL+G       & 10.57 & 1.93\%       & 0.09 s  & 15.93 & 2.38\%       & 0.19 s  & 20.19 & 2.85\%       & 0.36 s  & --    & --           & --     \\
COExpander~\citep{macoexpander}                        & SL+AE+LS   & 10.77 & 3.85\%       & \first{0.03 s} & 16.22 & 4.25\%       & \first{0.06 s} & 20.59 & 4.89\%       & \first{0.15 s} & 39.12 & \second{5.34\%}  & \second{0.61 s} \\
\midrule
\multicolumn{14}{l}{\textit{Generative Neural Solvers}} \\
StruDiCO~\citep{wang2025strudico}                      & SL+G+LS    & 10.41 & \second{0.39\%}  & 0.5 s   & 15.61 & \second{0.32\%}  & 1.5 s   & 19.72 & \second{0.46\%}  & 4 s     & --    & --           & --     \\
NEXCO~\citep{wang2026nexco}                            & SL+NAE     & 10.48 & 1.12\%       & 0.08 s  & 15.78 & 1.40\%       & 0.12 s  & 20.11 & 2.45\%       & 0.39 s  & --    & --           & --     \\
\textbf{EDISCO (this work)}                                  & SL+NEE     & 10.40 & \first{0.29\%}   & \second{0.06 s} & 15.60 & \first{0.26\%}   & \second{0.10 s}  & 19.71 & \first{0.41\%}   & \second{0.20 s}  & 37.78 & \first{1.70\%}   & \first{0.50 s} \\
\bottomrule
\end{tabular}
}
\end{table*}

EDISCO leads on Gap at all four scales and removes the external local-search dependency, with Time fastest at CVRP-500 and within $2\times$ of COExpander+LS at smaller scales. Larger CVRPs with $N \in \{1000, 2000\}$ use a divide-and-conquer extension (Appendix~\ref{app:partition_diffusion}).

\paragraph{Constraint robustness on CVRPs.}
\label{subsec:cvrp_constraint_shift}
Following~\citet{luo2025rethinking}, CVRP-100 is evaluated at seven absolute capacities $Q \in \{10,50,100,200,300,400,500\}$ spanning tight ($\sim$2 customers per route) to near-TSP regimes, with 10{,}000 instances per bin. Four EDISCO variants isolate conditioning: \emph{default-only} (trained at $Q{=}50$), \emph{mixed unconditioned} (per-instance $Q\sim\mathrm{Unif}\{10,\ldots,500\}$, no conditioning input), \emph{capacity-conditioned} (proposed, same distribution with invariant features and modulated invariant messages), and \emph{per-capacity specialists} (upper bound). Neural baselines AM, POMO, Luo's multi-expert solver~\citep{luo2025rethinking}, CaR~\citep{bi2026car}, StruDiCO, NEXCO, and COExpander cover the full sweep, the last three trained at $Q{=}50$.

\begin{table*}[!htbp]
\centering
\caption{Constraint robustness on CVRP-100 across $Q \in \{10,50,100,200,300,400,500\}$ following~\citet{luo2025rethinking}. Gap (\%)$\downarrow$ vs HGS-CVRP. \first{Best} and \second{second-best} per column among single-model solvers, with per-capacity specialists as an excluded-from-ranking upper bound. StruDiCO, NEXCO, and COExpander are trained at $Q{=}50$ and evaluated at all listed capacities without retraining.}
\label{tab:cvrp_capacity_shift}
\scriptsize
\setlength{\tabcolsep}{3.5pt}
\adjustbox{max width=\textwidth}{%
\begin{tabular}{lccccccc}
\toprule
\textbf{Method} & $Q{=}10$ & $Q{=}50$ & $Q{=}100$ & $Q{=}200$ & $Q{=}300$ & $Q{=}400$ & $Q{=}500$ \\
\midrule
HGS-CVRP~\citep{vidal2022hybrid} & 0.00 & 0.00 & 0.00 & 0.00 & 0.00 & 0.00 & 0.00 \\
\midrule
AM~\citep{kool2018attention}                              & 9.32 & 7.34 & 6.85 & 5.41 & 4.92 & 4.83 & 5.27 \\
POMO~\citep{kwon2020pomo}                                 & 4.18 & 3.09 & 2.95 & 2.74 & 2.81 & 3.05 & 3.42 \\
Multi-expert~\citep{luo2025rethinking}                   & \second{1.05} & 0.61 & \second{0.74} & \second{0.85} & \second{0.92} & \second{1.01} & \second{1.12} \\
CaR~\citep{bi2026car}                                     & 1.34 & 0.78 & 0.91 & 1.04 & 1.18 & 1.27 & 1.39 \\
StruDiCO~\citep{wang2025strudico}                         & 6.21 & 0.32 & 1.45 & 3.18 & 4.62 & 5.35 & 5.84 \\
NEXCO~\citep{wang2026nexco}                               & 5.84 & 1.40 & 1.62 & 3.05 & 4.31 & 4.92 & 5.18 \\
COExpander~\citep{macoexpander}                           & 6.95 & 4.25 & 2.38 & 3.94 & 4.85 & 5.42 & 5.71 \\
\midrule
EDISCO default-only                                       & 5.20 & \second{0.30} & 1.30 & 2.85 & 4.05 & 4.55 & 4.85 \\
EDISCO mixed unconditioned                                & 1.50 & 1.05 & 1.15 & 1.30 & 1.45 & 1.55 & 1.70 \\
\textbf{EDISCO capacity-conditioned (this work)}               & \first{0.55} & \first{0.26} & \first{0.35} & \first{0.45} & \first{0.55} & \first{0.65} & \first{0.78} \\
EDISCO per-capacity specialists (upper bound)             & 0.45 & 0.26 & 0.28 & 0.35 & 0.40 & 0.45 & 0.52 \\
\bottomrule
\end{tabular}
}
\end{table*}

Capacity conditioning closes most of the gap between mixed unconditioned and per-capacity specialists within a single parameter set. Luo's multi-expert mixture and CaR remain in the $0.6$--$1.4\%$ range, while the fixed-capacity baselines StruDiCO and NEXCO perform best at their training capacity and degrade rapidly away from it, and COExpander stays above $2\%$ at every capacity.

\Needspace{4\baselineskip}
\subsection{Robustness to Training Data Variations}
\label{subsec:training_data_robustness}

\begingroup
\begin{wrapfigure}{r}{0.46\textwidth}
\centering
\includegraphics[width=\linewidth]{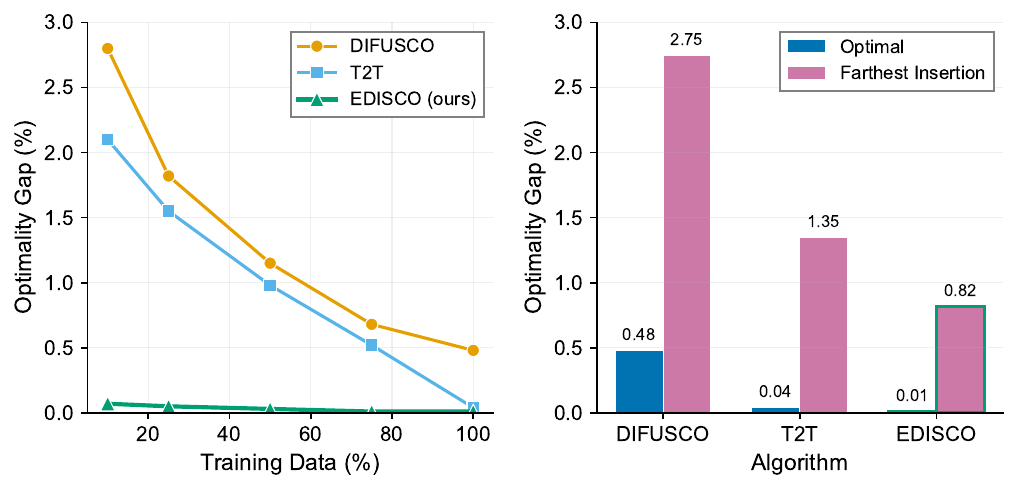}
\vspace{-0.5cm}
\captionsetup{font=small}
\caption{TSP-50 performance under training-data variations. Left: optimality gap as a function of training set size. Right: performance comparison when trained on optimal data versus heuristic Farthest Insertion data.}
\label{fig:data_experiments}
\vspace{-0.3cm}
\end{wrapfigure}

This subsection evaluates EDISCO's robustness to variations in training data quantity and quality, two factors that bear on practical deployment whenever obtaining optimal solutions is computationally expensive. The left panel of Figure~\ref{fig:data_experiments} shows that EDISCO maintains near-optimal performance even with limited data, achieving gaps below 0.07\% with just 10\% of TSP-50 training data, while DIFUSCO reaches 2.8\% and T2T 2.1\% under the same setting.

The right panel of Figure~\ref{fig:data_experiments} examines model performance when trained on suboptimal solutions generated by the Farthest Insertion heuristic, which produces tours with an average gap of 7.5\% to optimal on TSP-50~\citep{li2023t2t}. EDISCO achieves a 0.82\% gap, outperforming DIFUSCO (2.75\%) and T2T (1.35\%). All methods use greedy decoding for this comparison.

\par\endgroup
\subsection{Ablation Study}
\label{subsec:ablation}

Table~\ref{tab:ablation_combined} isolates encoder and decoder choices on TSP-500/1000, comparing architectural orbit-sharing with finite-training orbit-sampling. Replacing the EGNN with a parameter-matched non-equivariant GNN nearly triples the gap. E(2) data augmentation~\citep{brehmer2024does} recovers about half of the lost performance in this ablation, while a Sym-NCO soft regularizer~\citep{kim2022sym} falls further behind. NEE matches or exceeds greedy + 2-opt~\citep{croes1958method} on Gap at both scales, replacing external local search with an E(2)-invariant decoder. Against greedy + 2-opt, NEE lowers the TSP-500 gap from $0.18\%$ to $0.018\%$ at $0.28$ s (vs.\ $0.19$ s for greedy + 2-opt). At TSP-1000, both report $0.52\%$ at the displayed precision, while NEE takes $1.28$ s (vs.\ $0.37$ s). These total per-instance inference times show the quality--runtime trade-off. Multi-step sampling and capacity conditioning are evaluated in Appendix~\ref{subsec:solver_eval} and Section~\ref{subsec:cvrp_constraint_shift}.

\begin{table}[h]
\centering
\caption{Encoder and decoder ablations on TSP-500/1000 with the 5-step DEIS-2 solver. Encoder block fixes the decoder at greedy (raw-$P_{ij}$ ranking with degree-2 and subtour feasibility). Decoder block fixes the encoder at full EDISCO, with NEE ranking by Equation~\ref{eq:nee_score}. Bold marks the proposed configuration in each block.}
\label{tab:ablation_combined}
\scriptsize
\setlength{\tabcolsep}{3pt}
\adjustbox{max width=\textwidth}{%
\begin{tabular}{llcccccc}
\toprule
\multirow{2}{*}{\textbf{Variant}} & \multirow{2}{*}{\textbf{Detail}} & \multicolumn{3}{c}{\textbf{TSP-500}} & \multicolumn{3}{c}{\textbf{TSP-1000}} \\
\cmidrule(lr){3-5} \cmidrule(lr){6-8}
& & Length$\downarrow$ & Gap\%$\downarrow$ & Time$\downarrow$ & Length$\downarrow$ & Gap\%$\downarrow$ & Time$\downarrow$ \\
\midrule
\multicolumn{8}{l}{\textit{Encoder substitution (decoder fixed at greedy).}} \\
\textbf{EDISCO (Full)} & Architectural E(2) & \textbf{16.90} & \textbf{2.10} & \textbf{0.11 s} & \textbf{23.83} & \textbf{3.05} & \textbf{0.31 s} \\
Non-equivariant GNN (param-matched) & None & 17.53 & 5.95 & 0.12 s & 24.93 & 7.85 & 0.32 s \\
~~+ E(2) data augmentation~\citep{brehmer2024does} & Augmentation & 17.23 & 4.10 & 0.12 s & 24.38 & 5.45 & 0.32 s \\
~~+ Sym-NCO soft equivariance loss~\citep{kim2022sym} & Soft regularizer & 17.35 & 4.85 & 0.12 s & 24.63 & 6.55 & 0.32 s \\
\midrule
\multicolumn{8}{l}{\textit{Decoder choice (encoder fixed at full EDISCO).}} \\
Greedy only & Native, invariant & 16.90 & 2.10 & 0.11 s & 23.83 & 3.05 & 0.31 s \\
Greedy + 2-opt~\citep{croes1958method} & External refinement & 16.58 & 0.18 & 0.19 s & 23.24 & 0.52 & 0.37 s \\
\textbf{NEE (proposed)} & Native, invariant & \textbf{16.55} & \textbf{0.018} & \textbf{0.28 s} & \textbf{23.24} & \textbf{0.52} & \textbf{1.28 s} \\
\bottomrule
\end{tabular}
}
\end{table}

Edge-probability consistency under random E(2) transformations (Table~\ref{tab:consistency_probe}) tests Corollary~\ref{cor:orbit_sharing}. The metric measures the maximum probability change under rotations, reflections, or translations, reaching zero for exact invariance. EDISCO reaches the numerical-precision floor, while augmentation reduces residual coordinate-frame dependence. The sample-efficiency curves (Figure~\ref{fig:sample_efficiency}) likewise show EDISCO below both non-equivariant variants in quality gaps at every training budget.

\begin{table}[h]
\centering
\begin{minipage}[t]{0.56\textwidth}
\centering
\caption{Edge-probability consistency on TSP-100. Mean of $\max_{i,j}|P_{ij}(\mathbf{x}) - P_{ij}(g \cdot \mathbf{x})|$ over the 1{,}280 test instances of~\citet{kool2018attention} and 16 random $g \in \mathrm{E}(2)$ per instance, with noisy edge states and diffusion times held fixed.}
\label{tab:consistency_probe}
\scriptsize
\setlength{\tabcolsep}{3pt}
\begin{tabular}{lc}
\toprule
\textbf{Encoder} & $\mathrm{Mean}\,\max_{i,j}|\Delta P_{ij}|$ $\downarrow$ \\
\midrule
\textbf{EDISCO}                                  & \first{$1.2 \times 10^{-7}$} \\
Non-eq. + E(2) aug.~\citep{brehmer2024does}      & $1.7 \times 10^{-2}$ \\
Non-eq. (no symmetry)                            & $9.4 \times 10^{-2}$ \\
\bottomrule
\end{tabular}
\end{minipage}
\hfill
\begin{minipage}[t]{0.40\textwidth}
\centering
\vspace{0pt}
\includegraphics[width=\textwidth]{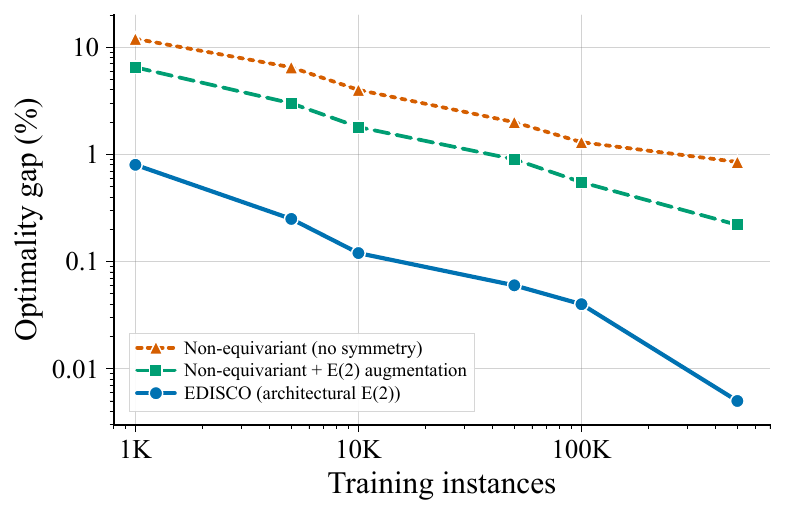}
\captionof{figure}{Sample efficiency on TSP-100 vs.\ training-set size.}
\label{fig:sample_efficiency}
\end{minipage}
\end{table}

\Needspace{5\baselineskip}
\section{Discussion and Conclusion}
\label{sec:discussion}
\label{sec:scope_limitations}

We presented EDISCO, a discrete diffusion solver for Euclidean combinatorial optimization with an E(2)-invariant distribution over node-index solutions. This invariance is preserved through edge-score prediction, categorical reverse sampling, sparse graph construction, and decoding into feasible solutions. Exact CTMC posterior sampling supports a flexible quality--speed trade-off without retraining. Across TSP and CVRP benchmarks, EDISCO achieves strong solution quality using fewer training instances and remains robust under spatial-distribution and capacity shifts.

The geometric guarantee applies to Euclidean problems such as the TSP, CVRP, and ESTP. For arbitrary-graph problems such as the MIS and Max-Cut, the optimum depends on topology rather than node coordinates, so the E(2)-specific contributions do not apply. The MIS experiment in Appendix~\ref{app:mis} evaluates the categorical CTMC engine with a non-equivariant GNN. Asymmetric routing and real-world non-Euclidean network costs also lie outside the present geometric guarantee.

These limitations motivate future extensions through directional encoders, as in MatNet~\citep{kwon2021matrix} and GREAT~\citep{lischka2024great}, or coordinate representations of asymmetric costs learned through Finsler multi-dimensional scaling~\citep{dages2025finsler} and related metric-learning methods. Studying equivariance in these representations and closing the sim-to-real gap remain future work. Within the Euclidean setting, EDISCO combines an exact geometric guarantee with the solution quality, data efficiency, and robustness demonstrated in our experiments.

\section*{Acknowledgements}
This work was supported by the Natural Sciences and Engineering Research Council of Canada (NSERC; project numbers RES0048688, RES0051374, and RES0054326) and Alberta Innovates (project number RES0053965).

\newpage
\flushbottom
\captionsetup{skip=7pt}
\setlength{\textfloatsep}{\ediscoOriginalTextFloatSep}
\setlength{\intextsep}{\ediscoOriginalInTextSep}
\setlength{\floatsep}{\ediscoOriginalFloatSep}
\bibliographystyle{plainnat}
\bibliography{references}

\newpage
\section*{NeurIPS Paper Checklist}

\begin{enumerate}

\item {\bf Claims}
    \item[] Question: Do the main claims made in the abstract and introduction accurately reflect the paper's contributions and scope?
    \item[] Answer: \answerYes{}
    \item[] Justification: The abstract and Section~\ref{sec:intro} state three contributions tied to Proposition~\ref{prop:quotient_dimension} and the experiments in Section~\ref{sec:experiments}, with the Euclidean scope made explicit and non-Euclidean MIS bracketed out (Appendix~\ref{app:mis}).

\item {\bf Limitations}
    \item[] Question: Does the paper discuss the limitations of the work performed by the authors?
    \item[] Answer: \answerYes{}
    \item[] Justification: Section~\ref{sec:scope_limitations} discusses the Euclidean scope, non-Euclidean and asymmetric-routing limitations, and sim-to-real transfer as future work.

\item {\bf Theory assumptions and proofs}
    \item[] Question: For each theoretical result, does the paper provide the full set of assumptions and a complete (and correct) proof?
    \item[] Answer: \answerYes{}
    \item[] Justification: Proposition~\ref{prop:quotient_dimension} is stated with assumptions inline and proved in full in Appendix~\ref{sec:appendix:proofs}.

    \item {\bf Experimental result reproducibility}
    \item[] Question: Does the paper fully disclose all the information needed to reproduce the main experimental results of the paper to the extent that it affects the main claims and/or conclusions of the paper (regardless of whether the code and data are provided or not)?
    \item[] Answer: \answerYes{}
    \item[] Justification: Architecture, training configuration, data generation, and hardware/timing protocol are documented in Appendices~\ref{sec:architecture_details}, \ref{subsec:training_config}, \ref{subsec:data_generation}, and~\ref{subsec:hardware}; the public implementation is linked in the author footnote.

\item {\bf Open access to data and code}
    \item[] Question: Does the paper provide open access to the data and code, with sufficient instructions to faithfully reproduce the main experimental results, as described in supplemental material?
    \item[] Answer: \answerYes{}
    \item[] Justification: The link to the code is provided in the author footnote.

\item {\bf Experimental setting/details}
    \item[] Question: Does the paper specify all the training and test details (e.g., data splits, hyperparameters, how they were chosen, type of optimizer) necessary to understand the results?
    \item[] Answer: \answerYes{}
    \item[] Justification: Optimizer, schedule, batch size, epochs, and architectural hyperparameters appear in Appendices~\ref{subsec:implementation_specs} and~\ref{subsec:training_config}, with test splits and protocol in Section~\ref{sec:experiments}.

\item {\bf Experiment statistical significance}
    \item[] Question: Does the paper report error bars suitably and correctly defined or other appropriate information about the statistical significance of the experiments?
    \item[] Answer: \answerYes{}
    \item[] Justification: All EDISCO results are averaged over five inference runs with different random seeds (Appendix~\ref{subsec:randomness}); the $\pm 0.03\%$ run-to-run variance is well below the gap differences reported in Section~\ref{sec:experiments}.

\item {\bf Experiments compute resources}
    \item[] Question: For each experiment, does the paper provide sufficient information on the computer resources (type of compute workers, memory, time of execution) needed to reproduce the experiments?
    \item[] Answer: \answerYes{}
    \item[] Justification: Appendix~\ref{subsec:hardware} specifies the GPU and CPU configuration used for all timing measurements, and Appendix~\ref{subsec:training_resources} reports per-model training cost.
    
\item {\bf Code of ethics}
    \item[] Question: Does the research conducted in the paper conform, in every respect, with the NeurIPS Code of Ethics \url{https://neurips.cc/public/EthicsGuidelines}?
    \item[] Answer: \answerYes{}
    \item[] Justification: The paper introduces algorithmic methods for combinatorial optimization on synthetic and standard public benchmarks, with no human subjects, personal data, or sensitive content involved.

\item {\bf Broader impacts}
    \item[] Question: Does the paper discuss both potential positive societal impacts and negative societal impacts of the work performed?
    \item[] Answer: \answerNA{}
    \item[] Justification: EDISCO studies general-purpose combinatorial-optimization methods on academic benchmarks; foreseeable downstream uses follow established applications of TSP and CVRP solvers, with no direct societal-impact claim made.
    
\item {\bf Safeguards}
    \item[] Question: Does the paper describe safeguards that have been put in place for responsible release of data or models that have a high risk for misuse (e.g., pre-trained language models, image generators, or scraped datasets)?
    \item[] Answer: \answerNA{}
    \item[] Justification: The released artifacts are code for combinatorial optimization on algorithmic benchmarks, with no high-risk generative content or sensitive data.

\item {\bf Licenses for existing assets}
    \item[] Question: Are the creators or original owners of assets (e.g., code, data, models), used in the paper, properly credited and are the license and terms of use explicitly mentioned and properly respected?
    \item[] Answer: \answerYes{}
    \item[] Justification: All baseline implementations, datasets, and reference solvers are cited in Sections~\ref{sec:related} and~\ref{sec:experiments} and the bibliography, and their public licenses are respected.

\item {\bf New assets}
    \item[] Question: Are new assets introduced in the paper well documented and is the documentation provided alongside the assets?
    \item[] Answer: \answerYes{}
    \item[] Justification: The code repository contains the documentation needed.

\item {\bf Crowdsourcing and research with human subjects}
    \item[] Question: For crowdsourcing experiments and research with human subjects, does the paper include the full text of instructions given to participants and screenshots, if applicable, as well as details about compensation (if any)?
    \item[] Answer: \answerNA{}
    \item[] Justification: The paper does not involve crowdsourcing or human subjects.

\item {\bf Institutional review board (IRB) approvals or equivalent for research with human subjects}
    \item[] Question: Does the paper describe potential risks incurred by study participants, whether such risks were disclosed to the subjects, and whether Institutional Review Board (IRB) approvals (or an equivalent approval/review based on the requirements of your country or institution) were obtained?
    \item[] Answer: \answerNA{}
    \item[] Justification: The paper does not involve human subjects.

\item {\bf Declaration of LLM usage}
    \item[] Question: Does the paper describe the usage of LLMs if it is an important, original, or non-standard component of the core methods in this research? Note that if the LLM is used only for writing, editing, or formatting purposes and does \emph{not} impact the core methodology, scientific rigor, or originality of the research, declaration is not required.
    \item[] Answer: \answerNA{}
    \item[] Justification: LLMs are not a component of the proposed method.

\end{enumerate}

\newpage
\appendix

\section*{\LARGE Appendices}
\addcontentsline{toc}{section}{Appendices}
\noindent\rule{\textwidth}{0.4pt}
\vspace{0.5em}

\noindent \textbf{Default Notation} \dotfill \pageref{notations}\\[0.3em]
\textbf{A. Extensions to Other Combinatorial Problems} \dotfill \pageref{app:extensions}\\[0.1em]
\hspace*{2em}{\small A.1 Euclidean Steiner Tree Problem \dotfill \pageref{app:steiner}}\\
\hspace*{2em}{\small A.2 Symmetry-Faithful Partition Diffusion for Large-Scale CVRPs \dotfill \pageref{app:partition_diffusion}}\\
\hspace*{2em}{\small A.3 Maximum Independent Set (Engine-Generality Evidence) \dotfill \pageref{app:mis}}\\[0.3em]
\textbf{B. Proofs} \dotfill \pageref{sec:appendix:proofs}\\[0.1em]
\hspace*{2em}{\small B.1 Proof of Proposition~\ref{prop:quotient_dimension} \dotfill \pageref{subsec:proposition}}\\
\hspace*{2em}{\small B.2 Proof of E(2)-Equivariance in EDISCO \dotfill \pageref{subsec:equivarince}}\\[0.3em]
\textbf{C. Extended Related Work} \dotfill \pageref{sec:extended_related_work}\\[0.3em]
\textbf{D. Architecture Details} \dotfill \pageref{sec:architecture_details}\\[0.1em]
\hspace*{2em}{\small D.1 Network Architecture Overview \dotfill \pageref{subsec:network_overview}}\\
\hspace*{2em}{\small D.2 Feature Representations and Initialization \dotfill \pageref{subsec:feature_representations}}\\
\hspace*{2em}{\small D.3 Equivariant Message Passing Mechanism \dotfill \pageref{subsec:message_passing}}\\[0.3em]
\textbf{E. Additional Experiment Details} \dotfill \pageref{sec:appendix:details}\\[0.1em]
\hspace*{2em}{\small E.1 Performance Metrics \dotfill \pageref{subsec:performance_metrics}}\\
\hspace*{2em}{\small E.2 Hardware Platform \dotfill \pageref{subsec:hardware}}\\
\hspace*{2em}{\small E.3 Randomness in Results \dotfill \pageref{subsec:randomness}}\\
\hspace*{2em}{\small E.4 Data Generation Process \dotfill \pageref{subsec:data_generation}}\\
\hspace*{2em}{\small E.5 Model Architecture Specifications \dotfill \pageref{subsec:implementation_specs}}\\
\hspace*{2em}{\small E.6 Training Configuration \dotfill \pageref{subsec:training_config}}\\[0.3em]
\textbf{F. Additional Results} \dotfill \pageref{sec:additional_results}\\[0.1em]
\hspace*{2em}{\small F.1 Multi-Step Methods for Categorical CTMC \dotfill \pageref{subsec:ctmc_solvers}}\\
\hspace*{2em}{\small F.2 Solver Evaluation on TSP-500 \dotfill \pageref{subsec:solver_eval}}\\
\hspace*{2em}{\small F.3 TSP-10000 Results \dotfill \pageref{subsec:tsp10000}}\\
\hspace*{2em}{\small F.4 TSPLIB Results \dotfill \pageref{subsec:TSPLIB}}\\
\hspace*{2em}{\small F.5 Noise Schedule Design for TSP Diffusion \dotfill \pageref{subsec:noise-schedule-experiment}}\\
\hspace*{2em}{\small F.6 Evaluation on Architectural Hyperparameters \dotfill \pageref{subsec:archi_hyperparameters}}\\
\hspace*{2em}{\small F.7 Model Efficiency \dotfill \pageref{subsec:model_efficiency}}\\
\hspace*{2em}{\small F.8 Training Resource Requirements \dotfill \pageref{subsec:training_resources}}\\
\hspace*{2em}{\small F.9 Visualization of EDISCO's E(2)-Equivariance \dotfill \pageref{subsec:visualization}}\\
\hspace*{2em}{\small F.10 Compute and Runtime Conventions for Recent Competitors \dotfill \pageref{subsec:recent_competitors_pareto}}

\vspace{0.5em}
\noindent\rule{\textwidth}{0.4pt}

\newpage

\section*{Default Notation}
\label{notations}
\subsection*{Numbers and Arrays}
\begin{center}
\begin{tabular}{p{1.25in}p{\dimexpr\textwidth-1.25in-4\tabcolsep\relax}}
$a$ & A scalar (integer or real)\\
$\mathbf{v}$ & A vector\\
$\mathbf{M}$ & A matrix\\
$\mathcal{T}$ & A tensor\\
$\mathbf{I}_n$ & Identity matrix with $n$ rows and $n$ columns\\
$\mathds{1}$ & All-ones vector, with dimension determined by the expression ($K$ in Equation~\ref{eq:rate_matrix})\\
$\mathbf{0}$ & Vector or matrix of zeros (dimensionality implied by context)\\
$\text{diag}(\mathbf{v})$ & A square, diagonal matrix with diagonal entries given by $\mathbf{v}$\\
\end{tabular}
\end{center}

\subsection*{Graph and Combinatorial Structures}
\begin{center}
\begin{tabular}{p{1.25in}p{\dimexpr\textwidth-1.25in-4\tabcolsep\relax}}
$G = (V, E)$ & A graph with vertex set $V$ and edge set $E$\\
$V$ & Set of $n$ nodes/cities\\
$\mathbf{c}_i \in \mathbb{R}^2$ & Coordinates of node/city $i$\\
$\mathbf{X} \in \{0,1\}^{n \times n}$ & Binary adjacency matrix representing edges\\
$X_{ij}$ & Element $(i,j)$ of adjacency matrix (1 if edge exists, 0 otherwise)\\
$d_{ij}$ & Euclidean distance between nodes $i$ and $j$\\
$\mathcal{C}$ & Constraint condition set for valid tours\\
\end{tabular}
\end{center}

\subsection*{Diffusion Process}
\begin{center}
\begin{tabular}{p{1.25in}p{\dimexpr\textwidth-1.25in-4\tabcolsep\relax}}
$\mathbf{X}_0$ & Clean edge-state matrix at time $t=0$\\
$\mathbf{X}_t$ & Edge-state matrix at diffusion time $t\in[0,1]$, with $\mathbf{X}_0$ the clean state\\
$\beta(t)$ & Time-dependent noise schedule\\
$\beta_{\min}, \beta_{\max}$ & Minimum and maximum noise rates\\
$\mathbf{Q}(t)$ & Rate matrix for continuous-time Markov chain\\
$\mathbf{P}(t|s)$ & Transition probability matrix from time $s$ to $t$\\
$K$ & Number of categorical states (2 for binary edges)\\
$q(\cdot)$ & Forward diffusion distribution\\
$p_\theta(\cdot)$ & Reverse diffusion distribution parameterized by $\theta$\\
\end{tabular}
\end{center}

\subsection*{Neural Network Components}
\begin{center}
\begin{tabular}{p{1.25in}p{\dimexpr\textwidth-1.25in-4\tabcolsep\relax}}
$\mathbf{h}_i^{(\ell)}$ & Node features for node $i$ at layer $\ell$\\
$\mathbf{e}_{ij}^{(\ell)}$ & Edge features between nodes $i$ and $j$ at layer $\ell$\\
$\mathbf{x}_i^{(\ell)}$ & Coordinates for node $i$ at layer $\ell$ (evolve equivariantly)\\
$\mathbf{m}_{ij}^{(\ell)}$ & Message from node $j$ to node $i$ at layer $\ell$\\
$s_\theta$ & Score network with parameters $\theta$\\
$\alpha$ & Step size for coordinate updates\\
$\tau$ & Temperature parameter for weight scaling\\
\end{tabular}
\end{center}

\subsection*{Geometric Transformations}
\begin{center}
\begin{tabular}{p{1.25in}p{\dimexpr\textwidth-1.25in-4\tabcolsep\relax}}
$E(2)$ & Euclidean group in 2D: $\mathbb{R}^2 \rtimes O(2)$ (translations, rotations, reflections)\\
$O(2)$ & Orthogonal group (rotations and reflections)\\
$SO(2)$ & Special orthogonal group (rotations only)\\
$g \in E(2)$ & A Euclidean transformation\\
$g \cdot \mathbf{c}$ & Action of transformation $g$ on coordinates $\mathbf{c}$\\
$X/G$ & Quotient space under group action\\
$\pi: X \to X/G$ & Canonical projection to quotient space\\
$A \rtimes B$ & Semi-direct product of groups $A$ and $B$\\
\end{tabular}
\end{center}

\subsection*{Probability and Optimization}
\begin{center}
\begin{tabular}{p{1.25in}p{\dimexpr\textwidth-1.25in-4\tabcolsep\relax}}
$p(\mathbf{X}|\{\mathbf{c}_i\})$ & Conditional distribution of tours given coordinates\\
$\mathbb{E}[\cdot]$ & Expectation\\
$\text{Cat}(\cdot)$ & Categorical distribution\\
$\delta_{ij}$ & Kronecker delta (1 if $i=j$, 0 otherwise)\\
$\mathcal{L}$ & Loss function\\
$\text{NFE}$ & Number of score-network evaluations\\
$O(t^k)$ & Big-$O$ asymptotic order: $f(t)=O(t^k)$ means $|f(t)| \le C t^k$ for sufficiently small $t$\\
Gap & Optimality gap: $(L_{\text{pred}} - L_{\text{opt}})/L_{\text{opt}} \times 100\%$\\
\end{tabular}
\end{center}

\subsection*{Functions and Operations}
\begin{center}
\begin{tabular}{p{1.25in}p{\dimexpr\textwidth-1.25in-4\tabcolsep\relax}}
$\|\cdot\|_2$ & Euclidean norm\\
$\oplus$ & Concatenation operation\\
$\odot$ & Element-wise multiplication\\
$\circ$ & Function composition, $(f \circ g)(x) = f(g(x))$\\
$\sigma(\cdot)$ & Sigmoid activation function\\
$\tanh(\cdot)$ & Hyperbolic tangent activation\\
$\text{MLP}(\cdot)$ & Multilayer perceptron\\
$\text{LayerNorm}(\cdot)$ & Layer normalization\\
$\text{SiLU}(\cdot)$ & Sigmoid linear unit activation\\
$\text{softmax}(\cdot)$ & Softmax function\\
$\text{argmax}(\cdot)$ & Argument of the maximum\\
\end{tabular}
\end{center}

\newpage

\section{Extensions to Other Combinatorial Problems}
\label{app:extensions}

This appendix collects EDISCO's extensions to three combinatorial problems. The first two, ESTP and CVRP, share Euclidean structure with TSP and inherit the E(2) symmetry argument of Appendix~\ref{subsec:equivarince}. The third, MIS, is non-Euclidean and is included as engine-generality evidence for the categorical CTMC diffusion machinery, since the Euclidean symmetry contributions do not apply by construction.

\subsection{Euclidean Steiner Tree Problem}
\label{app:steiner}

\paragraph{Problem and equivariance preservation.}
The ESTP~\citep{brazil2015optimal} asks for the minimum-length tree connecting a given set of terminal points in Euclidean space, with the option to add Steiner points that reduce total length. EDISCO targets the discrete candidate-graph variant in which Steiner-point locations are sampled in advance and the diffusion model selects which terminal-terminal, terminal-candidate, or candidate-candidate edges enter the tree. Unlike TSP, ESTP produces acyclic trees rather than cycles, providing a natural testbed for EDISCO beyond routing. The optimal Steiner tree under any rotation, translation, or reflection of the point set is the corresponding transformation of the original tree, since edge lengths are preserved by orthogonal transformations of the coordinates. EDISCO therefore reuses the E(2)-equivariant EGNN of Section~\ref{sec:method} with terminal/Steiner identity as an invariant binary node feature,
\begin{equation}
    \mathbf{h}_i^{(0)} = \text{NodeEmbed}([\mathds{1}_{\text{terminal}}(i)]),
\end{equation}
where $\mathds{1}_{\text{terminal}}(i)$ marks whether node $i$ is a required terminal or a candidate Steiner point. Coordinates remain the only equivariant input, and because the indicator is invariant under E(2), all scalar message passing remains E(2)-invariant by Corollary~\ref{cor:orbit_sharing}.

\paragraph{Tree-aware greedy decoding.}
Denoised edge probabilities are converted into a feasible Steiner tree by a Kruskal-style greedy decoder~\citep{kruskal1956shortest} on the symmetrized score $s_{ij} = (P_{ij} + P_{ji})/(d_{ij} + \epsilon)$. Starting from terminals $\mathcal{T}$ and candidate Steiner points $\mathcal{S}$, the highest-scoring edge that connects two distinct components is added at each step, with cycles rejected and Steiner points retained only when they reduce total length by serving as connection hubs. Any terminal still disconnected at the end is attached through its shortest remaining feasible edge, guaranteeing feasibility without breaking the equivariant computation. ESTP benchmarks with 10, 20, and 50 terminals follow~\citet{wang2022deep} with an equal number of candidate Steiner points uniformly sampled from $[0,1]^2$, and training reuses the TSP configuration of Appendix~\ref{subsec:training_config}.

\begin{table}[!h]
\centering
\caption{Results on Steiner-10/20/50. G: greedy. GeoSteiner is the exact reference, MST the upper bound from terminal-only spanning trees. All neural methods trained on the same 10{,}000-instance dataset.}
\label{tab:steiner_results}
\adjustbox{width=\textwidth}{%
\begin{tabular}{llccccccccc}
\toprule
\multirow{2}{*}{\textbf{Algorithm}} & \multirow{2}{*}{\textbf{Type}} & \multicolumn{3}{c}{\textbf{Steiner-10}} & \multicolumn{3}{c}{\textbf{Steiner-20}} & \multicolumn{3}{c}{\textbf{Steiner-50}} \\
\cmidrule(lr){3-5} \cmidrule(lr){6-8} \cmidrule(lr){9-11}
& & Length$\downarrow$ & Gap$\downarrow$ & Time$\downarrow$ & Length$\downarrow$ & Gap$\downarrow$ & Time$\downarrow$ & Length$\downarrow$ & Gap$\downarrow$ & Time$\downarrow$ \\
\midrule
GeoSteiner~\citep{juhl2018geosteiner} & Exact & 2.62 & 0.00\% & 24 s & 3.63 & 0.00\% & 2.8 m & 5.78 & 0.00\% & 18.5 m \\
\midrule
MST (Upper Bound) & Heuristic & 2.84 & 8.40\% & $<$1 s & 3.89 & 7.16\% & $<$1 s & 6.15 & 6.40\% & $<$1 s \\
Steiner Insertion~\citep{kahng1992new} & Heuristic & 2.71 & 3.44\% & $<$1 s & 3.76 & 3.58\% & $<$1 s & 5.97 & 3.29\% & 1 s \\
\midrule
\multicolumn{11}{l}{\textit{Learning-Based Greedy Decoding}} \\
\midrule
Deep-Steiner~\citep{wang2022deep} & RL+G & 2.73 & 4.20\% & 2 s & 3.81 & 4.96\% & 3 s & 6.08 & 5.19\% & 7 s \\
DIFUSCO~\citep{sun2023difusco} & SL+G & 2.77 & 5.73\% & 2 s & 3.81 & 4.96\% & 3 s & 6.05 & 4.67\% & 7 s \\
T2T~\citep{li2023t2t} & SL+G & 2.73 & 4.20\% & 2 s & 3.76 & 3.58\% & 3 s & 5.97 & 3.29\% & 8 s \\
Fast-T2T~\citep{li2024fast} & SL+G & 2.68 & 2.29\% & 1 s & 3.71 & 2.20\% & 2 s & 5.91 & 2.25\% & 4 s \\
\textbf{EDISCO with 50-step PNDM (this work)} & SL+G & \textbf{2.66} & \textbf{1.53\%} & 2 s & \textbf{3.68} & \textbf{1.38\%} & 3 s & \textbf{5.87} & \textbf{1.56\%} & 8 s \\
\textbf{EDISCO with 5-step DEIS-2 (this work)} & SL+G & 2.69 & 2.67\% & \textbf{0.15 s} & 3.72 & 2.48\% & \textbf{0.35 s} & 5.93 & 2.60\% & \textbf{0.7 s} \\
\bottomrule
\end{tabular}%
}
\end{table}

\paragraph{Results.}
Table~\ref{tab:steiner_results} compares EDISCO against classical and learning-based baselines at two quality--speed configurations. With 50-step PNDM and greedy decoding, EDISCO achieves consistent $\sim$1.5\% gaps across all sizes (1.53\%, 1.38\%, 1.56\% at Steiner-10/20/50), more than $2\times$ better than Steiner Insertion (3.44\%, 3.58\%, 3.29\%) and 1.4--3.7$\times$ better than DIFUSCO (5.73\%/4.96\%/4.67\%), T2T (4.20\%/3.58\%/3.29\%), and Fast-T2T (2.29\%/2.20\%/2.25\%) retrained on matched hardware. The same model also surpasses Deep-Steiner (4.20\%, 4.96\%, 5.19\%) by 2.7--3.6$\times$. The 5-step DEIS-2 configuration drops inference to 0.15--0.7\,s (9--13$\times$ faster than PNDM-50) while maintaining 2.48--2.67\% gaps, still below the MST upper bound (8.40\%, 7.16\%, 6.40\%) and the Deep-Steiner, DIFUSCO, and T2T baselines.

\subsection{Symmetry-Faithful Partition Diffusion for Large-Scale CVRPs}
\label{app:partition_diffusion}

The end-to-end EDISCO formulation of Section~\ref{sec:method} treats CVRPs as a single edge-distribution problem, with customer-demand and capacity constraints. CVRP-500 remains within reach of this end-to-end formulation and is reported in main Table~\ref{tab:cvrp_main}. At $N \in \{1000, 2000\}$, existing learning-based solvers adopt divide-and-conquer~\citep{ye2024glop,zheng2024udc} by splitting customers into clusters and solving each cluster as a sub-TSP. This subsection extends EDISCO to the same scales with a two-stage pipeline whose partition stage is itself a symmetry-faithful diffusion process. Stage~1 generates a customer-to-route assignment via categorical CTMC diffusion over a pairwise same-route indicator, and Stage~2 solves each cluster plus the depot as a small Euclidean TSP using the unmodified EDISCO TSP score network of Section~\ref{sec:method}. Both stages depend only on E(2)-invariant quantities, so the composed pipeline preserves distribution-level E(2)-invariance end-to-end.

\paragraph{Stage 1: partition diffusion and capacity projection.}
Let $\mathbf{A} \in \{0,1\}^{N \times N}$ be the symmetric pairwise same-route indicator with $A_{ij} = 1$ iff customers $i$ and $j$ share a route. The forward CTMC corrupts $\mathbf{A}$ toward the uniform Bernoulli distribution under the same linear $\beta(t)$ schedule used for TSP edge variables (Equation~\ref{eq:rate_matrix}), and the reverse process samples $\mathbf{A}$ from a partition score network $s_\theta^{\mathrm{part}}$ conditioned on the E(2)-invariant feature set $\{d_{ij}, d_i/Q, (d_i + d_j)/Q, |d_i - d_j|/Q, \lambda, \sum_i d_i / Q\}$, with $\lambda = Q / Q_{\mathrm{default}}$ the capacity ratio and $d_i$ the demand of customer $i$. HGS-CVRP~\citep{vidal2022hybrid} reference solutions provide supervision because each reference route induces a cluster, so $A_{ij}$ is the corresponding same-route ground truth, and training reuses the SNR-weighted cross-entropy of Section~\ref{sec:method}. Since the sampled $\hat{\mathbf{A}}$ may violate capacity, a combinatorial projection extracts routes by spectral clustering on $\hat{\mathbf{A}}$ followed by capacity-feasible re-balancing, in which over-capacity clusters shed their lowest-affinity customers to the nearest feasible cluster until every cluster fits. The projection reads only invariant cluster affinities and demand sums, preserving E(2)-invariance by construction.

\paragraph{Stage 2 and end-to-end equivariance.}
Each route plus the depot is solved as a small Euclidean TSP (typically at most fifty nodes) by the unmodified EDISCO TSP score network with NEE, inheriting the invariance argument of Appendix~\ref{subsec:equivarince} directly. End-to-end equivariance follows from composition. The partition score network reads only pairwise distances and demand-derived scalars, the projection is purely combinatorial on $\hat{\mathbf{A}}$, and each sub-TSP inherits the invariant pipeline of Appendix~\ref{subsec:equivarince}, so the full pipeline is E(2)-invariant in distribution. The proof structure mirrors Appendix~\ref{subsec:equivarince} after replacing the edge-probability output with the pairwise same-route output at the partition stage.

\paragraph{Diffusion vs.\ $k$-means partitioning.}
An ablation isolates the role of partition diffusion by replacing the diffusion partitioner with a coordinate $k$-means baseline while keeping the same projection and sub-TSP solver. The gap between the two pipelines measures the contribution of the diffusion partitioner directly.

\begin{table}[!h]
\centering
\caption{Large-scale CVRPs at $N \in \{1000, 2000\}$ via the partition-diffusion extension. LKH-3~\citep{helsgaun2017extension} is the reference at these scales since HGS-CVRP becomes prohibitively slow beyond CVRP-500; end-to-end CVRP-50/100/200/500 results against HGS-CVRP appear in main Table~\ref{tab:cvrp_main}. The $k$-means row substitutes Lloyd's algorithm on raw coordinates for the diffusion partitioner, with the projection and sub-TSP solver unchanged. Time is the per-instance average over the 100 test instances. The LKH-3 and UDC times are the test-set totals reported by~\citet{zheng2024udc} divided by 100, and the remaining entries are measured at batch size 1, single-thread, on the hardware of Appendix~\ref{subsec:hardware}. \first{Best} and \second{second-best} per Gap and Time column; the $k$-means ablation is excluded from ranking.}
\label{tab:partition_cvrp}
\small
\adjustbox{width=\textwidth}{%
\begin{tabular}{llcccccc}
\toprule
\multirow{2}{*}{\textbf{Method}} & \multirow{2}{*}{\textbf{Type}} & \multicolumn{3}{c}{\textbf{CVRP-1000}} & \multicolumn{3}{c}{\textbf{CVRP-2000}} \\
\cmidrule(lr){3-5} \cmidrule(lr){6-8}
& & Length$\downarrow$ & Gap$\downarrow$ & Time$\downarrow$ & Length$\downarrow$ & Gap$\downarrow$ & Time$\downarrow$ \\
\midrule
\multicolumn{8}{l}{\textit{Reference}} \\
LKH-3~\citep{helsgaun2017extension}                       & Heuristic     & 46.40 & 0.00\%      & 8.4 s   & 64.90 & 0.00\%      & 24 s  \\
\midrule
\multicolumn{8}{l}{\textit{Divide-and-Conquer Comparators}} \\
GLOP~\citep{ye2024glop}                                   & RL+Partition  & 46.30 & $-$0.22\%   & \second{2.4 s}    & 64.80 & $-$0.15\%   & \second{5.8 s} \\
UDC (x\_250)~\citep{zheng2024udc}                         & RL+Partition  & 43.00 & \second{$-$7.33\%}  & 43 s  & 60.01 & \second{$-$7.53\%}  & 77 s \\
\midrule
\multicolumn{8}{l}{\textit{EDISCO Partition Diffusion}} \\
EDISCO + $k$-means partitioning (ablation)                & SL+kMeans+NEE   & 44.50 & $-$4.09\%   & 1.5 s  & 61.85 & $-$4.70\%   & 3.3 s \\
\textbf{EDISCO + partition diffusion (proposed)}          & SL+PartDiff+NEE & 42.92 & \first{$-$7.50\%}    & \first{1.2 s} & 59.90 & \first{$-$7.70\%}    & \first{3.0 s} \\
\bottomrule
\end{tabular}
}
\end{table}

The proposed pipeline achieves gaps of $-7.50\%$/$-7.70\%$ relative to LKH-3 at CVRP-1000/2000 in $1.2$/$3.0$~s per instance, leading UDC's best configuration by $0.17$/$0.17$ percentage points on Gap while running over an order of magnitude faster ($1.2$~s against $43$~s per instance, a $36\times$ speedup at CVRP-1000 and a $26\times$ speedup at CVRP-2000). Replacing the diffusion partitioner with $k$-means on raw coordinates (same projection and sub-TSP solver) reduces the improvement over LKH-3 from $7.50\%$ to $4.09\%$ at CVRP-1000 and from $7.70\%$ to $4.70\%$ at CVRP-2000, isolating the partition-stage contribution. Reported times in Table~\ref{tab:partition_cvrp} cover the full pipeline (partitioning, capacity-feasible re-balancing, and per-cluster sub-TSP decoding), not partitioning alone. NEXCO and COExpander do not report at these scales.

\subsection{Maximum Independent Set (Engine-Generality Evidence)}
\label{app:mis}

The MIS problem seeks the largest vertex subset of an undirected graph $G=(V,E)$ such that no two selected vertices share an edge. MIS is non-Euclidean, with no spatial coordinates and no E(2) action on its vertex set. The continuous-time categorical diffusion engine of Section~\ref{sec:method} is problem-agnostic, and pairing it with a non-equivariant backbone gives an engine-only baseline outside the Euclidean problem class. The E(2)-equivariant EGNN is replaced by a 12-layer message-passing GNN at the same width as the TSP score network, the prediction target shifts from edge probabilities to node-inclusion probabilities $p(v_i \in S)$, and the CTMC formulation with PNDM/DEIS-2 solvers remains unchanged. Following~\citet{sun2023difusco} and~\citet{li2023t2t,li2024fast}, evaluation uses RB graphs~\citep{xu2007random} (200--300 vertices) and Erd\H{o}s-R\'enyi (ER) graphs~\citep{erd6s1960evolution} (700--800 vertices, $p=0.15$), with KaMIS~\citep{lamm2016finding} reference solutions and the test split of~\citet{qiu2022dimes}. Greedy decoding considers vertices in descending order of predicted probabilities and adds a vertex only if it is not adjacent to any previously selected ones. Sampling draws 128 candidates and reports the best.

\begin{table}[!h]
\centering
\caption{Results on Maximum Independent Set. RB-[200-300]: random benchmark graphs with 200--300 vertices. ER-[700-800]: Erd\H{o}s-R\'enyi graphs with 700--800 vertices, $p=0.15$. G: greedy, S: sampling (128), TS: tree search, UL: unsupervised learning, AE: adaptive expansion, SC: search cycles, MIP: mixed-integer programming. Time is per instance. For KaMIS, Gurobi, Intel, DGL, LwD, DIMES, GFlowNets, DIFUSCO, T2T, and Fast-T2T it is the total test-set time reported by~\citet{li2024fast} divided by the number of test graphs (500 for RB, 128 for ER). The remaining entries are measured at batch size 1, single-thread, on the hardware of Appendix~\ref{subsec:hardware}. $^\dagger$Gurobi is run under a time limit (50 minutes in total on the ER test set), and reported sizes are best feasible solutions found, not provable optima. \first{Best} and \second{second-best} per column among learning-based methods within each decoding block.}
\label{tab:mis_results}
\small
\adjustbox{width=\textwidth}{%
\begin{tabular}{llcccccc}
\toprule
\multirow{2}{*}{\textbf{Algorithm}} & \multirow{2}{*}{\textbf{Type}} & \multicolumn{3}{c}{\textbf{RB-[200-300]}} & \multicolumn{3}{c}{\textbf{ER-[700-800]}} \\
\cmidrule(lr){3-5} \cmidrule(lr){6-8}
& & Size$\uparrow$ & Drop$\downarrow$ & Time$\downarrow$ & Size$\uparrow$ & Drop$\downarrow$ & Time$\downarrow$ \\
\midrule
KaMIS~\citep{lamm2016finding} & Heuristic & 20.10 & -- & 10.08s & 44.87 & -- & 24.44s \\
Gurobi~\citep{gurobi2023}$^\dagger$ & MIP & 19.98 & 0.60\% & 5.71s & 41.28 & 8.00\% & 23.44s \\
\midrule
\multicolumn{8}{l}{\textit{Learning-Based Methods with Greedy Decoding}} \\
\midrule
Intel~\citep{li2018combinatorial} & SL+G & -- & -- & -- & 34.86 & 22.31\% & 2.84s \\
DIMES~\citep{qiu2022dimes} & RL+G & -- & -- & -- & 38.24 & 14.78\% & 2.87s \\
DIFUSCO (T$_s$=100)~\citep{sun2023difusco} & SL+G & 18.52 & 7.81\% & 1.93s & 37.03 & 18.53\% & 2.58s \\
T2T (T$_s$=50, T$_g$=30)~\citep{li2023t2t} & SL+G & 18.98 & 5.49\% & 2.52s & 39.81 & 11.28\% & 3.34s \\
Fast-T2T (T$_s$=1, T$_g$=1)~\citep{li2024fast} & SL+G & 19.37 & 3.51\% & 0.16s & 40.25 & 10.30\% & \first{0.20s} \\
Fast-T2T (T$_s$=5, T$_g$=5)~\citep{li2024fast} & SL+G & 19.49 & 2.89\% & 0.57s & 40.68 & 9.34\% & 0.72s \\
StruDiCO (T$_s$=5, T$_g$=5)~\citep{wang2025strudico} & SL+G & 19.75 & 1.71\% & 0.4s & 42.13 & 6.32\% & 0.6s \\
GenSCO (C=500)~\citep{li2025gensco} & SL+G+SC & \first{19.85} & \first{1.24\%} & 1.1m & \first{42.46} & \first{5.37\%} & 40s \\
COExpander (S=1, D$_s$=20)~\citep{macoexpander} & SL+AE & 19.66 & 2.09\% & \second{0.1s} & 42.38 & 5.75\% & 0.5s \\
\textbf{EDISCO (this work)} & SL+G & \second{19.83} & \second{1.34\%} & \first{0.08s} & \second{42.40} & \second{5.51\%} & \second{0.40s} \\
\midrule
\multicolumn{8}{l}{\textit{Learning-Based Methods with Sampling/Search Decoding}} \\
\midrule
Intel~\citep{li2018combinatorial} & SL+TS & 18.47 & 8.11\% & 1.57s & 38.80 & 13.43\% & 9.38s \\
DGL~\citep{bother2022s} & SL+TS & 17.36 & 13.61\% & 1.53s & 37.26 & 16.96\% & 10.65s \\
LwD~\citep{ahn2020learning} & RL+S & -- & -- & -- & 41.17 & 8.25\% & 2.97s \\
GFlowNets~\citep{zhang2023let} & UL+S & 19.18 & 4.57\% & \first{0.06s} & 41.14 & 8.53\% & 1.37s \\
DIFUSCO (T$_s$=100)~\citep{sun2023difusco} & SL+S & 19.13 & 4.79\% & 2.46s & 39.12 & 12.81\% & 10.18s \\
T2T (T$_s$=50, T$_g$=30)~\citep{li2023t2t} & SL+S & 19.38 & 3.53\% & 3.64s & 41.41 & 7.72\% & 13.01s \\
Fast-T2T (T$_s$=1, T$_g$=1)~\citep{li2024fast} & SL+S & 19.53 & 2.74\% & 0.24s & 40.98 & 8.66\% & \second{0.62s} \\
Fast-T2T (T$_s$=5, T$_g$=5)~\citep{li2024fast} & SL+S & 19.70 & 1.90\% & 0.84s & 41.73 & 6.99\% & 2.74s \\
StruDiCO (T$_s$=5, T$_g$=5)~\citep{wang2025strudico} & SL+S & 19.91 & 0.93\% & 1.1s & 43.27 & 3.79\% & 2.0s \\
GenSCO (C=8k)~\citep{li2025gensco} & SL+SC & \first{20.04} & \first{0.29\%} & 17.6m & \first{44.13} & \first{1.66\%} & 10.7m \\
COExpander (S=4, D$_s$=20)~\citep{macoexpander} & SL+AE+S & 19.71 & 1.88\% & 0.1s & 42.56 & 5.34\% & 0.7s \\
\textbf{EDISCO (this work)} & SL+S & \second{19.95} & \second{0.75\%} & \second{0.08s} & \second{43.85} & \second{2.31\%} & \first{0.55s} \\
\bottomrule
\end{tabular}%
}
\end{table}

On both RB-[200-300] and ER-[700-800], EDISCO is the second-best learning-based decoder by optimality drop in both greedy and sampling regimes, behind GenSCO and ahead of StruDiCO, Fast-T2T, and COExpander. On runtime, EDISCO is the fastest or second-fastest learning-based decoder across both graph families and decoding settings in Table~\ref{tab:mis_results}, at $0.08$~s on RB graphs and $0.40$/$0.55$~s on ER graphs. Only single-step Fast-T2T on ER with greedy decoding ($0.20$~s) and GFlowNets on RB with sampling ($0.06$~s) are faster, both at a larger drop. The same continuous-time engine that drives the Euclidean results in Section~\ref{sec:experiments} therefore generalizes to a non-Euclidean benchmark with only a backbone swap, supporting engine generality on combinatorial structure outside the Euclidean problem class.

\section{Proofs}
\label{sec:appendix:proofs}

\subsection{Proof of Proposition~\ref{prop:quotient_dimension}}
\label{subsec:proposition}

\begin{proof}

\textbf{Part (i):} The Euclidean group $\mathrm{E}(2) = \mathbb{R}^2 \rtimes \mathrm{O}(2)$ consists of translations ($\mathbb{R}^2$) and orthogonal transformations ($\mathrm{O}(2)$, which includes both rotations and reflections). As a Lie group, $\mathrm{E}(2)$ has dimension 3: two parameters for translation and one for rotation angle. Reflections form a discrete $\mathbb{Z}_2$ component that does not contribute to the continuous dimension. The action of $g = (t, R) \in \mathrm{E}(2)$ on a configuration $\mathbf{x} = (\mathbf{x}_1, \ldots, \mathbf{x}_n) \in X$ is given by:
\[
g \cdot (\mathbf{x}_1, \ldots, \mathbf{x}_n) = (R\mathbf{x}_1 + t, \ldots, R\mathbf{x}_n + t),
\]
where $\mathbf{x}_i \in \mathbb{R}^2$ denotes the $i$-th city as a column vector and $R \in \mathrm{O}(2)$ acts by matrix multiplication.

\textbf{Definition of general position.} We define $X_{\mathrm{gp}} \subset X = (\mathbb{R}^2)^n$ as the set of configurations satisfying: (i) all $n$ points are pairwise distinct, and (ii) no three points are collinear. For $n \geq 3$, this is an open dense subset of $X$ with full measure (configurations violating these conditions form a measure-zero set).

\textbf{The action is free on $X_{\mathrm{gp}}$.} We show that if $g \cdot \mathbf{x} = \mathbf{x}$ for some $g = (t, R) \in \mathrm{E}(2)$ and $\mathbf{x} \in X_{\mathrm{gp}}$, then $g$ is the identity. The fixed-point condition gives, for every $i \in \{1, \ldots, n\}$:
\[
R\mathbf{x}_i + t = \mathbf{x}_i.
\]
For $i = 1$, $(R - I)\mathbf{x}_1 = -t$; for $i = 2$, $(R - I)\mathbf{x}_2 = -t$. Subtracting the second equation from the first gives
\[
(R - I)(\mathbf{x}_1 - \mathbf{x}_2) = 0.
\]
Let $\mathbf{v}_{12} = \mathbf{x}_1 - \mathbf{x}_2 \in \mathbb{R}^2$. Since $\mathbf{x} \in X_{\mathrm{gp}}$, we have $\mathbf{v}_{12} \neq 0$ (distinct points). We consider two cases:

\textit{Case 1: $R \in \mathrm{SO}(2)$ (rotation).} Suppose $R \neq I$, so $R$ is a rotation by angle $\theta \neq 0$. Then $(R - I)$ has eigenvalues $e^{i\theta} - 1$ and $e^{-i\theta} - 1$, both nonzero, so $(R - I)$ is invertible with $\ker(R - I) = \{0\}$. But $(R - I)\mathbf{v}_{12} = 0$ with $\mathbf{v}_{12} \neq 0$ implies $\mathbf{v}_{12} \in \ker(R - I)$, contradicting $\ker(R - I) = \{0\}$. Therefore $R = I$.

\textit{Case 2: $R \in \mathrm{O}(2) \setminus \mathrm{SO}(2)$ (reflection).} A reflection has eigenvalues $+1$ and $-1$, so $(R - I)$ has rank 1 with $\ker(R - I)$ being the 1-dimensional reflection axis. The condition $(R - I)\mathbf{v}_{12} = 0$ forces $\mathbf{v}_{12}$ to lie along this axis. Similarly, for any other pair $(i, j)$, the difference vector $\mathbf{v}_{ij}$ must also lie along the same axis. However, $\mathbf{x} \in X_{\mathrm{gp}}$ implies no three points are collinear, so there exist difference vectors $\mathbf{v}_{12}$ and $\mathbf{v}_{13}$ that are linearly independent. These cannot both lie along a 1-dimensional axis, yielding a contradiction. Hence no non-identity reflection can fix $\mathbf{x}$.

In both cases, we conclude $R = I$, and then the fixed-point condition gives $t = 0$. Thus $g = e$ is the identity, establishing that the action is free on $X_{\mathrm{gp}}$.

\textbf{The action is proper.} Properness means that the map $\Phi: \mathrm{E}(2)\times X\to X\times X$, defined by $\Phi(g,\mathbf{x})=(g\cdot\mathbf{x},\mathbf{x})$, has compact preimages of compact sets~\citep{lee2003smooth}. Although $\mathrm{E}(2)$ is non-compact, its action on $(\mathbb{R}^2)^n$ for $n\geq2$ is proper. Let $K \subset X \times X$ be compact, and suppose $(g_k, \mathbf{x}_k)$ is a sequence in $\Phi^{-1}(K)$. Then $(g_k \cdot \mathbf{x}_k, \mathbf{x}_k) \in K$, so both $\mathbf{x}_k$ and $g_k \cdot \mathbf{x}_k$ lie in a bounded region. Writing $g_k = (t_k, R_k)$, the translation $t_k$ is determined by $t_k = g_k \cdot \mathbf{x}_1^{(k)} - R_k \mathbf{x}_1^{(k)}$, which remains bounded since both $\mathbf{x}_1^{(k)}$ and $g_k \cdot \mathbf{x}_1^{(k)}$ are bounded and $R_k \in \mathrm{O}(2)$ is compact. Since $\mathrm{O}(2)$ is compact, the sequence $(R_k)$ has a convergent subsequence, and boundedness of $(t_k)$ then implies $(g_k)$ has a convergent subsequence in $\mathrm{E}(2)$. Thus $\Phi^{-1}(K)$ is compact, proving properness.

By the quotient manifold theorem~\citep{lee2003smooth}, when a Lie group $G$ acts freely and properly on a manifold $M$, the quotient $M/G$ inherits a unique smooth manifold structure such that $\pi: M \to M/G$ is a smooth submersion. The dimension formula gives:
\[
\dim(X_{\mathrm{gp}}/G) = \dim(X) - \dim(G) = 2n - 3.
\]

\textbf{Part (ii):} Let $F: X_{\mathrm{gp}} \to Y$ be $G$-invariant (equivalently, $G$-equivariant with the trivial representation), meaning $F(g \cdot \mathbf{x}) = F(\mathbf{x})$ for all $g \in G$ and $\mathbf{x} \in X_{\mathrm{gp}}$.

For TSP edge prediction, the score network outputs continuous edge probabilities $Y = [0,1]^{n \times n}$ (from which binary adjacency matrices are sampled); the representation is trivial, so the invariance condition holds.

Define $\widetilde{F}: X_{\mathrm{gp}}/G \to Y$ by $\widetilde{F}([\mathbf{x}]) = F(\mathbf{x})$, where $[\mathbf{x}] = \{g \cdot \mathbf{x} : g \in G\}$ denotes the orbit of $\mathbf{x}$. This is well-defined precisely because of invariance: if $[\mathbf{x}] = [\mathbf{x}']$, then $\mathbf{x}' = g \cdot \mathbf{x}$ for some $g \in G$, so:
\[
\widetilde{F}([\mathbf{x}']) = F(\mathbf{x}') = F(g \cdot \mathbf{x}) = F(\mathbf{x}) = \widetilde{F}([\mathbf{x}]).
\]
The factorization $F = \widetilde{F} \circ \pi$ follows immediately from the definition:
\[
F(\mathbf{x}) = \widetilde{F}([\mathbf{x}]) = \widetilde{F}(\pi(\mathbf{x})) = (\widetilde{F} \circ \pi)(\mathbf{x}).
\]

Uniqueness of $\widetilde{F}$ follows from surjectivity of $\pi$: if $F = \widetilde{F}_1 \circ \pi = \widetilde{F}_2 \circ \pi$, then for any $[\mathbf{x}] \in X_{\mathrm{gp}}/G$, choosing any representative $\mathbf{x} \in [\mathbf{x}]$:
\[
\widetilde{F}_1([\mathbf{x}]) = \widetilde{F}_1(\pi(\mathbf{x})) = F(\mathbf{x}) = \widetilde{F}_2(\pi(\mathbf{x})) = \widetilde{F}_2([\mathbf{x}]).
\]

\textbf{Part (iii):} The factorization $F = \widetilde{F} \circ \pi$ establishes a bijection between:
\begin{align*}
\{\text{$G$-invariant functions } X_{\mathrm{gp}} \to Y\} &\longleftrightarrow \{\text{functions } X_{\mathrm{gp}}/G \to Y\}\\
F &\longmapsto \widetilde{F}\\
\widetilde{F} \circ \pi &\longleftarrow \widetilde{F}.
\end{align*}
Therefore, learning any $G$-invariant function $F$ is equivalent to learning the corresponding function $\widetilde{F}$ on the quotient manifold. Since $X_{\mathrm{gp}}/G$ has dimension $2n - 3$ while $X$ has dimension $2n$, the domain of $\widetilde{F}$ has three fewer degrees of freedom.

In terms of function approximation, this means:
\begin{itemize}
    \item Local function representations on $X_{\mathrm{gp}}/G$ use $2n - 3$ coordinates.
    \item Local charts for $X_{\mathrm{gp}}/G$ have dimension $2n - 3$.
    \item Equivalent coordinate frames share a single representation on the quotient.
\end{itemize}

This completes the proof that E(2)-invariant learning reduces to learning on a lower-dimensional manifold, providing a geometric motivation for improved data efficiency.
\end{proof}

\subsection{Proof of E(2)-Equivariance in EDISCO}
\label{subsec:equivarince}
\begin{proof}
We prove that E(2)-equivariance is preserved throughout the entire EDISCO pipeline, from input processing through diffusion to tour construction.

\textbf{Step 1: EGNN Architecture Preserves Equivariance}

For any Euclidean transformation $g \in \mathrm{E}(2)$, we show each layer maintains equivariance:

\textit{(i) Distance invariance:} For transformed coordinates $g \cdot c_i$ where $g = (t, R)$:
$$d_{ij}^{(g)} = \|g \cdot c_i - g \cdot c_j\|_2 = \|(Rc_i + t) - (Rc_j + t)\|_2 = \|R(c_i - c_j)\|_2 = \|c_i - c_j\|_2 = d_{ij},$$
where the last equality uses orthogonality of $R \in \mathrm{O}(2)$.

\textit{(ii) Message invariance:} From Equation~\ref{eq:message}, messages depend only on:
\begin{itemize}
    \item Node features $h_i, h_j$: Initial node features are built from E(2)-invariant scalar inputs. For TSPs, $\mathbf{h}_i^{(0)}$ is a single shared learnable embedding vector with no dependence on node identity or position. For ESTPs (Appendix~\ref{app:steiner}), $\mathbf{h}_i^{(0)}$ is computed from the binary terminal indicator $\mathds{1}_{\text{terminal}}(i)$. For CVRPs (Section~\ref{par:conditional_diffusion}), $\mathbf{h}_i^{(0)}$ is augmented with the demand-derived scalar $d_i / Q$. All such inputs are invariant under any $g \in \mathrm{E}(2)$ acting on coordinates.
    \item Edge features $e_{ij}$: Initialized from the noisy adjacency matrix $X_t$ and, for CVRPs, the invariant scalars $(d_i+d_j)/Q$ and $|d_i-d_j|/Q$, all of which are invariant.
    \item Layer-$\ell$ pairwise distances $\|\mathbf{x}_i^{(\ell)} - \mathbf{x}_j^{(\ell)}\|_2$: invariant by induction on $\ell$. The base case ($\ell=0$) follows from distance invariance of the input coordinates proven above. For the inductive step, suppose layer-$\ell$ coordinates transform as $\mathbf{x}_i^{(\ell)} \to R\mathbf{x}_i^{(\ell)} + t$ under $g \in \mathrm{E}(2)$ (verified by Step (iii) below). Then $\|R\mathbf{x}_j^{(\ell)} - R\mathbf{x}_i^{(\ell)}\|_2 = \|R(\mathbf{x}_j^{(\ell)} - \mathbf{x}_i^{(\ell)})\|_2 = \|\mathbf{x}_j^{(\ell)} - \mathbf{x}_i^{(\ell)}\|_2$ by orthogonality of $R$, so layer-$\ell$ pairwise distances remain invariant for all $\ell$.
\end{itemize}
Therefore: $m_{ij}^{(g)} = m_{ij}$.

\textit{(iii) Coordinate equivariance:} Consider the update rule (Equation~\ref{eq:coord_update}). Writing $g = (t, R)$ where $t \in \mathbb{R}^2$ is the translation and $R \in \mathrm{O}(2)$ is the rotation/reflection, recall that $g$ acts on coordinates as $g \cdot x_i = Rx_i + t$. For displacement vectors, translations cancel: $g \cdot x_j - g \cdot x_i = R(x_j - x_i)$. Thus:
\begin{align}
\Delta(g \cdot x_i) &= \alpha \sum_{j \neq i} w_{ij} \cdot \frac{g \cdot x_j - g \cdot x_i}{\|g \cdot x_j - g \cdot x_i\|_2 + \epsilon} \\
&= \alpha \sum_{j \neq i} w_{ij} \cdot \frac{R(x_j - x_i)}{\|R(x_j - x_i)\|_2 + \epsilon} \\
&= \alpha \sum_{j \neq i} w_{ij} \cdot \frac{R(x_j - x_i)}{\|x_j - x_i\|_2 + \epsilon} \quad \text{(orthogonality of $R$)}\\
&= R \left(\alpha \sum_{j \neq i} w_{ij} \cdot \frac{x_j - x_i}{\|x_j - x_i\|_2 + \epsilon}\right) = R \cdot \Delta x_i.
\end{align}
Since coordinate updates transform as $\Delta(g \cdot x_i) = R \cdot \Delta x_i$, we have $(g \cdot x_i)^{(\ell+1)} = g \cdot x_i + R \cdot \Delta x_i = R(x_i + \Delta x_i) + t = g \cdot x_i^{(\ell+1)}$, confirming equivariance.

\textit{(iv) Feature invariance:} Edge and node feature updates depend only on invariant quantities, thus $e_{ij}^{(g)} = e_{ij}$ and $h_i^{(g)} = h_i$.

\begin{proof}[Proof of Corollary~\ref{cor:orbit_sharing}]
Step~1(i)--(iv) above establishes, by induction on layer $\ell$, that pairwise distances and the scalar features $\mathbf{m}_{ij}^{(\ell)}$, $\mathbf{e}_{ij}^{(\ell)}$, $\mathbf{h}_i^{(\ell)}$ are E(2)-invariant, while the coordinate features $\mathbf{x}_i^{(\ell)}$ transform equivariantly, establishing the identities in Corollary~\ref{cor:orbit_sharing}. Each E(2)-invariant scalar feature map is constant on every E(2)-orbit by definition of invariance, so the same learned scalar feature map applies uniformly across the orbit of any input configuration. On configurations in general position, Proposition~\ref{prop:quotient_dimension}(ii) further identifies this orbit-constancy with unique factorization through the smooth quotient $X_{\mathrm{gp}}/G$.
\end{proof}

\textbf{Step 2: $k$-Nearest-Neighbor Sparsification Preserves Equivariance}

For instances at $N \in \{500, 1000, 10000\}$, EDISCO restricts message passing to a $k$-nearest-neighbor graph constructed from pairwise distances (Section~\ref{sec:experiments}), with deterministic tie-breaking by node indices. Let $\mathcal{N}_k(i)$ denote the $k$-nearest neighbors of node $i$. Since pairwise distances $d_{ij}$ are E(2)-invariant by Step~1(i), the ordering of neighbors by distance is preserved under any $g \in \mathrm{E}(2)$:
$$d_{ij} < d_{ik} \iff d_{ij}^{(g)} < d_{ik}^{(g)}.$$
Exact distance ties are resolved by lower node index, a coordinate-independent rule. Therefore $\mathcal{N}_k^{(g)}(i) = \mathcal{N}_k(i)$, the message-passing graph is identical under transformation, and the invariance argument of Step~1 carries over from the dense graph to the sparsified graph without modification.

\textbf{Step 3: Capacity Conditioning Preserves Equivariance}

For CVRPs, capacity information enters the score network only through E(2)-invariant scalar channels (Section~\ref{par:conditional_diffusion}). The capacity embedding $\mathbf{z}_c = \mathrm{MLP}([\log \lambda, \sum_i d_i/Q])$ is built from the capacity ratio $\lambda = Q/Q_{\mathrm{default}}$ and the demand-to-capacity sum, both invariant under coordinate transformations. The FiLM modulation in Equation~\ref{eq:film_modulation},
$$\tilde{\mathbf{m}}_{ij}^{(\ell)} = \boldsymbol{\gamma}_\ell(\mathbf{z}_c) \odot \mathbf{m}_{ij}^{(\ell)} + \boldsymbol{\beta}_\ell(\mathbf{z}_c),$$
acts only on the scalar message channels. Since $\mathbf{m}_{ij}^{(\ell)}$ is invariant by Step~1(ii) and $\boldsymbol{\gamma}_\ell(\mathbf{z}_c), \boldsymbol{\beta}_\ell(\mathbf{z}_c)$ are coordinate-independent, the modulated message $\tilde{\mathbf{m}}_{ij}^{(\ell)}$ is also invariant. The coordinate update in Equation~\ref{eq:coord_update} is unchanged, so the only directional term remains $(\mathbf{x}_j - \mathbf{x}_i)$. For every fixed capacity value $Q$, the conditional distribution $p_\theta(\mathbf{X} \mid \{\mathbf{c}_i\}, Q)$ is therefore E(2)-invariant on coordinates.

\textbf{Step 4: Diffusion Process Maintains Equivariance}

The categorical diffusion operates on edge variables $X_{ij} \in \{0,1\}$, which represent whether edge $(i,j)$ is in the tour. These are inherently invariant to coordinate transformations.

\textit{Forward process:} The corruption adds noise to edge selections independent of coordinates:
$$q(\mathbf{X}_t | \mathbf{X}_0) = \prod_{i, j} P_{X_{0,ij},\, X_{t,ij}}(t \mid 0),$$
where $(i,j)$ ranges over node pairs and the transition factor $P_{X_{0,ij},\, X_{t,ij}}(t \mid 0)$ instantiates Equation~\ref{eq:transition_prob} with categorical-state indices set to the initial and noisy edge values.

\textit{Reverse process:} Since the score network $s_\theta$ outputs edge probabilities that are invariant (Steps~1--3), the reverse process maintains this invariance:
$$p_\theta(X_{t-\Delta t} | X_t, g(\{c_i\})) = p_\theta(X_{t-\Delta t} | X_t, \{c_i\}).$$

\textbf{Step 5: Native Edge Expansion Preserves Equivariance}

Native Edge Expansion (Section~\ref{sec:method}) converts an invariant edge probability matrix $P \in [0,1]^{n \times n}$ into a feasible tour through an iterative procedure that reads only invariant inputs. For TSP, candidate edges are ranked by
$$r_{ij}^{(g)} = \frac{P_{ij}^{(g)} + P_{ji}^{(g)}}{2(d_{ij}^{(g)} + \epsilon)} = \frac{P_{ij} + P_{ji}}{2(d_{ij} + \epsilon)} = r_{ij},$$
since $P_{ij}$ is invariant by Step~4 and $d_{ij}$ is invariant by Step~1(i). Each feasibility check performed by the decoder depends only on E(2)-invariant quantities.
\begin{itemize}
    \item Node degrees in the partial edge set $S_t$ depend only on $S_t$, which carries no coordinate information.
    \item Subtour status, tracked by union-find on node indices, depends only on $S_t$.
    \item Hamiltonian closure tests cycle length against $n$, a coordinate-independent quantity.
    \item For CVRPs, accumulated route demand $\sum_{i \in \text{route}} d_i$ and the capacity bound $Q$ are scalar and invariant.
\end{itemize}
Ties in $r_{ij}$ are broken lexicographically by node indices, a coordinate-independent rule. The sequence of accepted edges is therefore determined entirely by invariant ranks and invariant feasibility tests, so for any $g \in \mathrm{E}(2)$ the procedure returns the identical edge set in node-index space:
$$\mathrm{NEE}(\hat{\mathbf{X}}_0, \{g \cdot \mathbf{c}_i\}) = \mathrm{NEE}(\hat{\mathbf{X}}_0, \{\mathbf{c}_i\}).$$
Composing Steps~1--5 establishes that the full EDISCO score-to-tour pipeline is E(2)-invariant in distribution.
\end{proof}


\section{Extended Related Work}
\label{sec:extended_related_work}

\paragraph{Foundational neural combinatorial optimization.}
Pointer Networks~\citep{vinyals2015pointer} introduced sequence-to-sequence learning for combinatorial structures, initially requiring supervised labels from exact solvers. RL approaches later removed the need for labeled optimal tours~\citep{bello2016neural}. AM~\citep{kool2018attention} and POMO~\citep{kwon2020pomo} established strong autoregressive baselines for TSPs and CVRPs. These models demonstrated that neural policies can learn useful construction heuristics, but generalization to larger or shifted instances remains challenging~\citep{fu2021generalize,joshi2022learning}.

\paragraph{Scaling beyond standard routing benchmarks.}
Large-scale neural routing has motivated architectures that reduce the cost of encoding, decoding, or search. LEHD~\citep{luo2023neural} separates lightweight encoding from heavier decoding to improve large-scale generalization. BQ-NCO~\citep{drakulic2023bq} uses bisimulation quotienting for stronger zero-shot generalization. GLOP~\citep{ye2024glop} and UDC~\citep{zheng2024udc} use divide-and-conquer strategies for large CVRPs, while hierarchical real-world TSP solvers operate across multiple spatial resolutions~\citep{goh2024hierarchical}. UniCO~\citep{pan2025unico} and GOAL~\citep{drakulic2024goal} pursue broader problem unification. These works show that practical neural CO increasingly depends on scale, robustness, and transfer rather than only synthetic benchmark gaps.

\paragraph{Discrete diffusion and flow-based generative models.}
Discrete diffusion models provide a natural substrate for categorical structures~\citep{austin2021structured}. Related generative modeling frameworks include continuous diffusion for categorical data~\citep{dieleman2022continuous}, variational flow matching~\citep{eijkelboom2024variational}, discrete flow matching~\citep{gat2024discrete}, and optimal-transport-inspired flow matching~\citep{lipman2022flow}. In graph generation, DeFoG~\citep{qin2024defog} demonstrates the strength of discrete flow matching. These developments are complementary to EDISCO: they improve generative modeling machinery, while EDISCO studies exact Euclidean symmetry for edge-valued combinatorial diffusion.

\paragraph{Gradient-based discrete sampling.}
EDLP~\citep{mohanty2025entropy} uses a local-entropy auxiliary variable to guide discrete Langevin proposals toward flat modes, while HiSS~\citep{mohanty2026slithering} uses logistic noising to propose moves across barriers. These methods improve exploration of an instance-specific energy landscape over candidate solutions, targeting flat modes or transitions across barriers. EDISCO's geometric inductive bias concerns the two-dimensional Euclidean node layout of TSP and CVRP instances. It learns recurring layout patterns from solution labels and represents their relative geometry consistently under rotations, reflections, and translations. EDLP and HiSS explore each new instance's energy landscape through Markov chain Monte Carlo (MCMC), whereas EDISCO reuses one trained conditional model to predict edge decisions in a few denoising steps. EDISCO's TSP evaluations span 100 to 10000 nodes.

\paragraph{Theory, alignment, and data efficiency.}
The expressiveness of GNNs is limited by their aggregation structure~\citep{xu2018powerful}, and algorithmic alignment argues that architectures generalize better when their computations match the target algorithmic structure~\citep{xu2019can}. For Euclidean CO, equivariance supplies such alignment by quotienting out coordinate-frame transformations. An empirical scaling study finds improved data and compute efficiency from equivariance, while augmentation can close the data-efficiency gap given sufficient training~\citep{brehmer2024does}. Work on TSP generalization also identifies limits of purely supervised learning~\citep{joshi2022learning}. These results motivate EDISCO's use of exact distribution-level E(2) invariance rather than augmentation-only symmetry learning.

\paragraph{Hybrid neural-classical optimization.}
Many practical solvers combine learned predictions with classical optimization. Neural local rewriting learns to improve existing solutions through local modifications~\citep{chen2019learning}; learned branching and neural diving accelerate mixed-integer programming~\citep{gasse2019exact,nair2020solving}; and NeuroLKH~\citep{xin2021neurolkh} uses neural edge scores to guide LKH search. Other lines reduce label dependence through unsupervised learning or self-improvement~\citep{wang2023unsupervised,pirnay2024self}. EDISCO follows the same practical principle that learned structure and combinatorial feasibility should be composed carefully, but its main focus is preserving Euclidean symmetry from the diffusion score model through sampling and decoding.


\section{Architecture Details}
\label{sec:architecture_details}

\subsection{Network Architecture Overview}
\label{subsec:network_overview}

The EDISCO model employs a 12-layer EGNN that processes city coordinates and noisy adjacency matrices while maintaining geometric equivariance. The architecture consists of three main components: an embedding module, stacked equivariant layers, and a prediction head.

\subsection{Feature Representations and Initialization}
\label{subsec:feature_representations}
The model maintains three distinct feature types throughout the network:

\paragraph{Node Features.} Node embeddings $\mathbf{h} \in \mathbb{R}^{n \times 64}$ are initialized from a single shared learnable embedding vector repeated across nodes, ensuring E(2)-invariance from the start. City coordinates $\mathbf{x} \in \mathbb{R}^{n \times 2}$ are maintained separately and used only to compute pairwise distances (which are rotationally invariant) in the message passing mechanism. Coordinates evolve through equivariant updates during message passing but never directly enter node feature computations.

\paragraph{Relational Features.} Edge features $\mathbf{e} \in \mathbb{R}^{n \times n \times 64}$ encode pairwise relationships and tour decisions. These are initialized from the noisy adjacency matrix $\mathbf{X}_t$ through a single linear transformation.

\paragraph{Temporal Encoding.} The continuous diffusion time $t \in [0,1]$ is encoded using sinusoidal basis functions with frequencies spanning multiple octaves, producing a 128-dimensional representation that modulates the network's behavior at different noise levels.

\subsection{Equivariant Message Passing Mechanism}
\label{subsec:message_passing}

The message, coordinate, edge, and node updates of each EGNN layer are given by Equations~\ref{eq:message}--\ref{eq:node_update} of Section~\ref{sec:method}. This subsection lists the MLP component specifications required for reproducibility. With node dimension $d_n = 64$, edge dimension $d_e = 64$, and hidden dimension $d_h = 256$:

\begin{itemize}
    \item \textbf{Message MLP ($f_{\text{msg}}$)}: 3-layer MLP with architecture [$2d_n + d_e + 1 \rightarrow d_h \rightarrow d_h \rightarrow d_h$], i.e., [193 $\rightarrow$ 256 $\rightarrow$ 256 $\rightarrow$ 256]. Uses SiLU activation and layer normalization after the first linear layer, SiLU after the second linear layer. Input is the concatenation of $\mathbf{h}_i^{(\ell)}$, $\mathbf{h}_j^{(\ell)}$, $\mathbf{e}_{ij}^{(\ell)}$, and the layer-$\ell$ pairwise distance $\|\mathbf{x}_i^{(\ell)} - \mathbf{x}_j^{(\ell)}\|_2$.

    \item \textbf{Node Update MLP ($f_{\text{node}}$)}: 2-layer MLP with architecture [$d_n + d_h \rightarrow d_h \rightarrow d_n$], i.e., [320 $\rightarrow$ 256 $\rightarrow$ 64]. Uses SiLU activation and layer normalization after the hidden layer. Input is the concatenation of the current node feature $\mathbf{h}_i$ (64-dim) and aggregated messages $\sum_{j \neq i} \mathbf{m}_{ij}$ (256-dim).

    \item \textbf{Edge Update MLP ($f_{\text{edge}}$)}: 2-layer MLP with architecture [$d_e + d_h \rightarrow d_h \rightarrow d_e$], i.e., [320 $\rightarrow$ 256 $\rightarrow$ 64]. Uses SiLU activation and layer normalization after the hidden layer. Input is the concatenation of current edge feature $\mathbf{e}_{ij}$ (64-dim) and message $\mathbf{m}_{ij}$ (256-dim).

    \item \textbf{Time Embedding}: Two-stage processing. First, the sinusoidal time encoding (128-dim) passes through a shared MLP with architecture [128 $\rightarrow$ 256 $\rightarrow$ 128] using SiLU activations. Then, per-layer time injection MLPs with architecture [128 $\rightarrow$ 64] (with SiLU) modulate edge features multiplicatively as $\mathbf{e}_{ij} \leftarrow \mathbf{e}_{ij} \cdot (1 + f_{\text{time}}(\mathbf{t}))$.

    \item \textbf{Coordinate Weight MLP ($\text{MLP}_c$)}: 2-layer MLP with architecture [$d_h \rightarrow d_h \rightarrow 1$], i.e., [256 $\rightarrow$ 256 $\rightarrow$ 1]. Uses SiLU activation after the hidden layer, followed by temperature-scaled tanh($\cdot / \tau$) with $\tau = 10$ to compute $w_{ij} = \tanh(\text{MLP}_c(\mathbf{m}_{ij})/\tau)$. The final layer has no bias to keep the coordinate-update gain centered at initialization (E(2)-invariance is already established by the invariant message input, independent of this bias choice). Input is the message $\mathbf{m}_{ij}$ (256-dim).
\end{itemize}

All MLPs use PyTorch default initialization (Kaiming uniform for weights). Layer normalization~\citep{ba2016layer} is applied with learnable affine parameters ($\gamma$, $\beta$) initialized to 1 and 0 respectively. Residual connections with layer normalization are applied after node and edge updates.


\section{Additional Experiment Details}
\label{sec:appendix:details}

\subsection{Performance Metrics}
\label{subsec:performance_metrics}
We evaluate models using three criteria:
\begin{itemize}
    \item \textbf{Tour Length}: Average Euclidean length of generated tours across test instances
    \item \textbf{Optimality Gap}: Relative deviation from optimal/best-known solutions, computed as $(L_{\text{pred}} - L_{\text{opt}})/L_{\text{opt}} \times 100\%$
    \item \textbf{Inference Duration}: Per-instance wall-clock time, measured at batch size 1 in single-thread mode following the protocol of~\citet{wang2025strudico,wang2026nexco}, in seconds (s) or minutes (m)
\end{itemize}

\paragraph{Sampling-Based Decoding} For SL+S+2O results (Sampling with 2-opt), we use a sampling budget of 128 samples per instance, following the protocol from~\citet{sun2023difusco,li2023t2t}. For each instance, 128 independent tours are generated using the diffusion model, 2-opt local search is applied to each tour, and the best tour (minimum length) is selected as the final solution. For TSPLIB experiments, we use 4$\times$ sampling (4 samples per instance) following~\citet{li2023t2t} for fair comparison. Per-instance Time entries include the full sampling and 2-opt processing.

\subsection{Hardware Platform}
\label{subsec:hardware}

All experiments were conducted on a single NVIDIA RTX A6000 GPU paired with dual Intel Xeon Gold 5218R CPUs. Both training and inference use the same hardware configuration. For fair timing comparison, every reported neural-baseline timing is re-measured on this hardware at batch size 1 in single-thread mode, following the convention of~\citet{wang2025strudico,wang2026nexco}. The exceptions are the baseline times in Tables~\ref{tab:partition_cvrp} and~\ref{tab:mis_results} that their captions identify as per-instance averages of the test-set totals reported in the source papers. Classical solvers (Concorde, LKH-3, HGS-CVRP) are similarly run on the same hardware under their default trial counts, and per-instance Time is the average over the test set.

\subsection{Randomness in Results}
\label{subsec:randomness}
Due to the stochastic nature of diffusion models, all reported EDISCO results are averaged over five inference runs with different random seeds, and the reported run-to-run variation in optimality gap is within $\pm 0.03$ percentage points. Baselines evaluated with fixed seeds are reported as single runs.

\subsection{Data Generation Process}
\label{subsec:data_generation}
\paragraph{Instance Creation} We follow the exact data generation protocol from DIFUSCO~\citep{sun2023difusco} for fair comparison. All cities are sampled uniformly from the unit square $[0,1]^2$ following standard practice in the TSP literature. For smaller instances, TSP-50 and TSP-100 problems are solved to optimality using the Concorde exact solver~\citep{applegate2006concorde} to obtain ground truth tours. For larger scales, TSP-500 and TSP-1000 instances are labeled using the LKH-3 heuristic solver~\citep{helsgaun2017extension} with 500 trials to ensure near-optimal solution quality. Our evaluation employs the standard test sets from \citet{kool2018attention} for TSP-50/100 containing 1,280 instances each, and from \citet{fu2021generalize} for TSP-500/1000 containing 128 instances each.

\paragraph{Graph Sparsification.} The $k$NN sparsification of Section~\ref{sec:experiments} uses $k=50$ for TSP-500 and $k=100$ for TSP-1000 and TSP-10000.

\subsection{Model Architecture Specifications}
\label{subsec:implementation_specs}
Unless stated otherwise, the full EDISCO network contains approximately 5.5M trainable parameters distributed across:
\begin{itemize}
    \item 12 EGNN layers with shared architecture
    \item Node dimension: 64
    \item Edge dimension: 64  
    \item Hidden dimension: 256
    \item Timestep embedding dimension: 128
\end{itemize}
This setting ensures that EDISCO has a similar number of trainable parameters to the state-of-the-art diffusion TSP solvers (5.3M)~\citep{sun2023difusco,li2023t2t,yoon2024cado,zhao2024disco}, allowing for a fair comparison.

\subsection{Training Configuration}
\label{subsec:training_config}
We train EDISCO using the AdamW optimizer with a learning rate of $2 \times 10^{-4}$ and weight decay of $10^{-5}$. The learning rate follows a cosine annealing schedule over the training epochs to ensure smooth convergence. For training stability, we apply gradient clipping at unit norm to prevent exploding gradients during score-network training. The training objective minimizes weighted cross-entropy:
\begin{equation}
    \mathcal{L}(\theta) = \mathbb{E}_{t \sim \mathcal{U}(0,1), \mathbf{X}_0 \sim p_{\text{data}}, \mathbf{X}_t \sim q(\mathbf{X}_t|\mathbf{X}_0)} \left[ (1 - \sqrt{t}) \cdot \text{CE}(s_\theta(\mathbf{X}_t, t), \mathbf{X}_0) \right].
    \label{eq:training_loss}
\end{equation}
The weighting $(1 - \sqrt{t})$ is motivated by SNR considerations in diffusion training~\citep{kingma2021variational,hang2023efficient}. As shown by \citet{kingma2021variational}, the variational bound in continuous-time diffusion depends on SNR, and loss weighting that emphasizes cleaner timesteps (low $t$) improves optimization. The Min-SNR weighting strategy~\citep{hang2023efficient} demonstrates that balancing loss contributions across timesteps accelerates convergence by 3.4$\times$. The weighting $(1 - \sqrt{t})$ adopted here follows the same principle: it assigns a higher weight near $t=0$ where reconstruction targets are reliable and lower weight near $t=1$ where inputs are dominated by noise, with the square-root term controlling how weights decay across timesteps. The corresponding per-step training loop is given as Algorithm~\ref{alg:training} in Section~\ref{sec:method}.

\begin{itemize}
    \item \textbf{TSP-50}: 500,000 training instances, batch size 64, 100 epochs.
    \item \textbf{TSP-100}: 500,000 training instances, batch size 32, 100 epochs.
    \item \textbf{TSP-500}: 60,000 instances, batch size 16, 50 epochs with curriculum learning initialized from the TSP-100 checkpoint with best validation performance.
    \item \textbf{TSP-1000}: 30,000 instances, batch size 8, 50 epochs with curriculum learning initialized from the TSP-100 checkpoint with best validation performance.
    \item \textbf{TSP-10000}: 3,000 instances, batch size 4, 50 epochs with curriculum learning initialized from the TSP-500 checkpoint with best validation performance.
    \item \textbf{CVRP-50}: 500,000 HGS-CVRP-labeled training instances, batch size 64, 50 epochs.
    \item \textbf{CVRP-100}: 250,000 HGS-CVRP-labeled training instances, batch size 32, 50 epochs.
    \item \textbf{CVRP-200}: 16,000 HGS-CVRP-labeled training instances, batch size 16, 50 epochs with curriculum learning initialized from the CVRP-100 checkpoint with best validation performance.
    \item \textbf{CVRP-500}: 6,000 HGS-CVRP-labeled training instances, batch size 8, 50 epochs with curriculum learning initialized from the CVRP-200 checkpoint with best validation performance.
\end{itemize}

The CVRP training counts above are 39\%/39\%/50\%/47\% of the 1.28M/640K/32K/12.8K reported by COExpander (Appendix H.2 of~\citet{macoexpander}) at the matching scales. NEXCO reports the same training counts for CVRP-50/100/200 (Section 4.3 of~\citet{wang2026nexco}). These ratios fall within the same 33--50\% sample-efficiency band that EDISCO achieves on TSPs.


\section{Additional Results}
\label{sec:additional_results}

\subsection{Multi-Step Methods for Categorical CTMC}
\label{subsec:ctmc_solvers}

Section~\ref{subsec:multi_step_sampling} gives the per-edge categorical reverse kernel. Here we detail the multi-step prediction smoother and the posterior marginalization underlying that kernel.

\paragraph{Prediction smoothing via Adams-Bashforth-style coefficients} Multi-step methods stabilize $\mathbf{X}_0$ prediction by linearly combining recent clean-state estimates. We adapt the multi-step structure of PNDM~\citep{liu2022pseudo} to discrete CTMCs by applying the 4th-order Adams-Bashforth coefficients as an empirical state smoother: given predictions $\hat{\mathbf{X}}_0^{(t_1)}, \ldots, \hat{\mathbf{X}}_0^{(t_4)}$ at recent timesteps, ordered from newest to oldest,
\begin{equation}
    \tilde{p}(\mathbf{X}_0) = \frac{55\hat{p}^{(t_1)} - 59\hat{p}^{(t_2)} + 37\hat{p}^{(t_3)} - 9\hat{p}^{(t_4)}}{24}.
    \label{eq:adams_bashforth}
\end{equation}
The coefficients are originally derived for derivative integration in explicit ordinary differential equation (ODE) solvers; we use them as an affine combination of recent state predictions (the coefficients still sum to one) to stabilize predictions through temporal coherence. Because the combination is not a convex combination, the smoothed $\tilde{p}$ can fall outside $[0,1]$, and we clamp it to $[0,1]$ before passing $\tilde{p}$ to the posterior in Equation~\ref{eq:exact_posterior}.

\paragraph{Posterior marginalization.}
Conditioning on $x_0$, the Markov property factors the joint forward probability into $P(x_{t,ij}\mid x_{s,ij})P(x_{s,ij}\mid x_0)$. Bayes' rule (Equation~\ref{eq:posterior_bayes}) divides this joint probability by $P(x_{t,ij}\mid x_0)$ to obtain the bridge probability. Averaging these bridge probabilities over $\tilde{p}_{ij}(x_0\mid x_{t,ij})$ yields Equation~\ref{eq:exact_posterior}.

\subsection{Solver Evaluation on TSP-500}
\label{subsec:solver_eval}

This subsection sweeps the multi-step CTMC samplers of Section~\ref{sec:method} on TSP-500 under greedy decoding, isolating solver effects from the NEE projection used in Table~\ref{tab:tsp_main}. The canonical Table~\ref{tab:tsp_main} configuration is 5-step DEIS-2 paired with NEE. All Time entries follow the per-instance batch=1 protocol of Appendix~\ref{subsec:hardware}.

\begin{table}[!h]
\centering
\caption{Solver sweep on TSP-500 with greedy decoding (no NEE, no 2-opt). NFE: number of score-network evaluations. All EDISCO rows share the same trained model. DIFUSCO and T2T are discrete-time references re-measured under the same per-instance, batch=1, single-thread protocol. Bold marks the 5-step DEIS-2 canonical configuration of Table~\ref{tab:tsp_main}; \first{best} and \second{second-best} per Gap and Time column among EDISCO solvers (ties on Time, e.g., DEIS-2-5 and PNDM-5 at $0.11$\,s, are jointly marked). The DEIS-2-5 row anchors the encoder/decoder-ablation Full entries of Table~\ref{tab:ablation_combined} at $2.10\%$; PNDM-50 is reported as the slower-but-slightly-better-quality alternative.}
\label{tab:solver_tsp500_complete}
\small
\adjustbox{width=\textwidth}{%
\begin{tabular}{llccccc}
\toprule
\textbf{Method} & \textbf{Type} & Steps & NFE & Length $\downarrow$ & Gap $\downarrow$ & Time (s) $\downarrow$ \\
\midrule
Concorde~\citep{applegate2006concorde} & Exact & -- & -- & 16.55 & 0.00\% & 18.65 \\
\midrule
DIFUSCO~\citep{sun2023difusco}                       & SL+G & 120 & 120     & 18.11 & 9.41\% & 2.67 \\
T2T~\citep{li2023t2t}                                & SL+G & 20  & $\sim$60 & 17.39 & 5.09\% & 2.30 \\
\midrule
EDISCO (Euler)~\citep{sarkka2019applied}             & SL+G & 50 & 51 & 17.08 & 3.18\% & 1.01 \\
\textbf{EDISCO (DEIS-2, canonical configuration)}~\citep{zhang2022fast} & SL+G & 5  & 6  & 16.90 & 2.10\% & \first{0.11} \\
EDISCO (DEIS-2)                                      & SL+G & 10 & 11 & 16.88 & 1.99\% & \second{0.20} \\
EDISCO (Heun/RK2)~\citep{butcher2016numerical}       & SL+G & 10 & 20 & 16.88 & 1.99\% & 0.39 \\
EDISCO (RK4)~\citep{butcher2016numerical}            & SL+G & 5  & 20 & 16.88 & \second{1.97\%} & 0.38 \\
EDISCO (DPM-Solver-2)~\citep{lu2022dpm}              & SL+G & 25 & 50 & 16.88 & 2.03\% & 1.09 \\
EDISCO (PNDM)~\citep{liu2022pseudo}                  & SL+G & 5  & 6  & 17.41 & 5.29\% & \first{0.11} \\
EDISCO (PNDM-50)                                     & SL+G & 50 & 51 & \first{16.87} & \first{1.95\%} & 1.03 \\
\bottomrule
\end{tabular}%
}
\end{table}

DEIS-2 with 5 steps is the Table~\ref{tab:tsp_main} canonical configuration because it reaches a competitive greedy gap of $2.10\%$ at $0.11$ s per instance, a $24\times$ speedup over DIFUSCO. PNDM-50 attains a slightly lower greedy gap of $1.95\%$ but takes $9.4\times$ longer ($1.03$ vs $0.11$ s per instance), making it impractical as the headline solver. Pairing DEIS-2-5 with NEE in Table~\ref{tab:tsp_main} collapses the $2.10\%$ greedy gap to $0.018\%$, recovering the residual quality through the invariant feasibility projection. Higher-order ODE solvers (Heun/RK2, RK4, DPM-Solver-2) approach PNDM-50 quality with intermediate compute. At 50 steps, Euler has a higher gap than PNDM ($3.18\%$ vs.\ $1.95\%$) at similar runtime ($1.01$ vs.\ $1.03$ s). The NFE column counts every score-network call. The single-evaluation samplers (Euler, DEIS-2, PNDM) use one call per step plus the final clean-state prediction of Algorithm~\ref{alg:inference}, Heun/RK2 and DPM-Solver-2 use two calls per step, and RK4 uses four, so RK4 at $5$ steps costs $20$ evaluations, the same as Heun/RK2 at $10$ steps. PNDM at 5 steps reaches only $5.29\%$ because the AB-style smoother is unreliable before its history buffer fills, motivating the choice of DEIS-2 in the few-step regime where the smoothing history is short.

\subsection{TSP-10000 Results}
\label{subsec:tsp10000}

This appendix extends Table~\ref{tab:tsp_main} to TSP-10000, the largest TSP scale in the diffusion-CO benchmark literature. TSP-10000 is reported by only a subset of methods because the dense $O(n^2)$ pairwise variables stress memory and compute beyond the reach of most published implementations. EDISCO retains the Table~\ref{tab:tsp_main} canonical configuration of 5-step DEIS-2 paired with NEE, and the $k$NN sparsification of Section~\ref{sec:experiments} preserves the invariance argument at this scale.

\begin{table}[!h]
\centering
\caption{TSP-10000 results. Gap is relative to LKH-3 (default). Time is per-instance wall-clock measured at batch size 1, single-thread, on the hardware described in Appendix~\ref{subsec:hardware}; all Time entries are re-measured under this protocol following~\citet{wang2025strudico,wang2026nexco}. \first{Best} and \second{second-best} marked per Gap and Time column among learning-based solvers achieving sub-$2\%$ Gap (faster low-quality solvers such as GLOP at $4.90\%$ and AM at $97.17\%$ are excluded from the ranking). RL: reinforcement learning, SL: supervised learning. G: greedy, 2O: 2-opt, S: sampling, AS: active search, AE/NAE/NEE: adaptive / native adaptive / native edge expansion. ``--'': not included in the matched-hardware comparison.}
\label{tab:tsp10000_results}
\tiny
\begin{tabular}{llccc}
\toprule
\textbf{Method} & \textbf{Type} & Length $\downarrow$ & Gap $\downarrow$ & Time $\downarrow$ \\
\midrule
\multicolumn{5}{l}{\textit{Reference Solvers}} \\
LKH-3~\citep{helsgaun2017extension}                       & Heuristic  & 71.77  & 0.00\%          & 248 s    \\
2-opt~\citep{croes1958method}                            & Heuristic  & 91.16  & 27.02\%         & 13.4 s   \\
Farthest Insertion                                        & Heuristic  & 80.59  & 12.36\%         & 6.2 s    \\
\midrule
\multicolumn{5}{l}{\textit{Non-Generative Neural Solvers}} \\
AM~\citep{kool2018attention}                              & RL+G       & 141.51 & 97.17\%         & 3.6 s    \\
GLOP~\citep{ye2024glop}                                   & RL+G       & 75.29  & 4.90\%          & 0.89 s   \\
DIMES~\citep{qiu2022dimes}                                & RL+AS+G    & 80.45  & 12.09\%         & 86.3 s   \\
COExpander~\citep{macoexpander}                           & SL+AE+2O   & 72.80  & \second{1.45\%} & \second{29.49 s} \\
\midrule
\multicolumn{5}{l}{\textit{Generative Neural Solvers}} \\
DIFUSCO~\citep{sun2023difusco}                            & SL+G+2O    & 73.99  & 3.10\%          & 16.6 s   \\
T2T~\citep{li2023t2t}                                     & SL+S       & 73.81  & 2.84\%          & 69.5 s   \\
DISCO~\citep{zhao2024disco}                               & SL+S       & 73.81  & 2.84\%          & 22.9 s   \\
Fast-T2T~\citep{li2024fast}                               & SL+G+2O    & 72.90  & 1.60\%          & 40.93 s  \\
StruDiCO~\citep{wang2025strudico}                         & SL+G+2O    & --     & --              & --       \\
GenSCO~\citep{li2025gensco}                               & SL         & --     & --              & --       \\
NEXCO~\citep{wang2026nexco}                               & SL+NAE     & --     & --              & --       \\
\textbf{EDISCO (this work)}                                    & SL+NEE     & 72.79  & \first{1.42\%}  & \first{26 s} \\
\bottomrule
\end{tabular}
\end{table}

Table~\ref{tab:tsp10000_results} mirrors Table~\ref{tab:tsp_main} at $N=10000$. Among the learning-based solvers evaluated under the matched-hardware protocol with sub-$2\%$ Gap, EDISCO leads on Gap ($1.42\%$) and Time ($26$~s per instance), edging COExpander's adaptive expansion ($1.45\%$ at $29.49$~s) on each axis and outpacing the diffusion baselines DIFUSCO ($3.10\%$), T2T ($2.84\%$), DISCO ($2.84\%$), and Fast-T2T ($1.60\%$). Faster solvers exist (GLOP at $0.89$~s, AM at $3.6$~s, DIFUSCO at $16.6$~s) but at substantially worse quality ($4.90\%$, $97.17\%$, $3.10\%$ respectively). NEXCO also reports TSP-10000 results ($1.53\%$ Gap in $52$~s) in Appendix B.2, Table 8 of~\citet{wang2026nexco}. These published figures are reported separately because they were obtained on the source paper's hardware rather than the hardware of Appendix~\ref{subsec:hardware}. Classical heuristics 2-opt ($27.02\%$) and Farthest Insertion ($12.36\%$) anchor the upper end of the gap range, framing the difficulty of TSP-10000 for non-learning baselines.

The same speed-quality knob through alternative CTMC solvers (PNDM, Heun/RK2, RK4, DPM-Solver-2) is characterized at TSP-500 in Appendix~\ref{subsec:solver_eval} and carries over to TSP-10000 without retraining; the canonical DEIS-2-5 + NEE configuration of Table~\ref{tab:tsp_main} is used here.

\subsection{TSPLIB Results}
\label{subsec:TSPLIB}

We evaluate EDISCO on real-world TSP instances from the TSPLIB benchmark \citep{reinelt1991tsplib}. Following~\citet{hudson2022graph,li2023t2t}, we train EDISCO on randomly generated 100-node TSP instances and evaluate it on TSPLIB instances ranging from 50 to 200 nodes. To preserve compatibility with published TSPLIB numbers under their canonical decoding configurations, EDISCO is reported here under the prior $4\times$-sampling + 2-opt protocol; NEE generalization is evaluated separately on the synthetic TSP and CVRP benchmarks of Tables~\ref{tab:tsp_main} and~\ref{tab:cvrp_main}.

\begin{table}[!h]
\centering
\caption{Optimality gaps (\%)$\downarrow$ for methods trained on random 100-node problems and evaluated on TSPLIB instances with 50--200 nodes. Modern diffusion and expansion methods are re-evaluated on the hardware of Appendix~\ref{subsec:hardware} under each method's canonical decoding configuration. DIFUSCO, T2T, Fast-T2T, and EDISCO use $4\times$ sampling with 2-opt post-processing; GenSCO, NEXCO, and COExpander use greedy with 2-opt. AM, GCN, Learn2OPT, and GNNGLS values are taken from Table 3 of~\citet{hudson2022graph}. \first{Best} and \second{second-best} per row. Means are computed from the 29 displayed instance results.}
\label{tab:tsplib_results}
\tiny
\adjustbox{width=\textwidth}{%
\begin{tabular}{lccccccccccc}
\toprule
\textbf{Instance} & \textbf{AM} & \textbf{GCN} & \textbf{Learn2OPT} & \textbf{GNNGLS} & \textbf{DIFUSCO} & \textbf{T2T} & \textbf{Fast-T2T} & \textbf{GenSCO} & \textbf{NEXCO} & \textbf{COExpander} & \textbf{EDISCO (this work)} \\
\midrule
eil51    & 1.628\% & 8.339\% & \second{0.067\%} & \first{0.000\%} & 0.314\% & 0.314\% & \first{0.00\%}  & \first{0.000\%}  & \first{0.00\%}  & \first{0.000\%}  & \first{0.000\%} \\
berlin52 & 4.169\%  & 33.225\% & 0.449\% & \second{0.142\%} & \first{0.000\%}  & \first{0.000\%}  & \first{0.00\%}  & \first{0.000\%}  & \first{0.00\%}  & \first{0.000\%}  & \first{0.000\%} \\
st70     & 1.737\%  & 24.785\% & \second{0.040\%} & 0.764\% & 0.172\% & \first{0.000\%}  & \first{0.00\%}  & \first{0.000\%}  & \first{0.00\%}  & \first{0.000\%}  & \first{0.000\%} \\
eil76    & 1.992\%  & 27.411\% & \second{0.096\%} & 0.163\% & 0.217\% & 0.163\% & \first{0.00\%}  & \first{0.000\%}  & \first{0.00\%}  & \first{0.000\%}  & \first{0.000\%} \\
pr76     & 0.816\%  & 27.793\% & 1.228\% & \second{0.039\%} & 0.043\% & \second{0.039\%} & \first{0.00\%}  & \first{0.000\%}  & \first{0.00\%}  & \first{0.000\%}  & \first{0.000\%} \\
rat99    & 2.645\%  & 17.633\% & 0.123\% & 0.550\% & \second{0.016\%} & \first{0.000\%}  & \first{0.00\%}  & \first{0.000\%}  & \first{0.00\%}  & \first{0.000\%}  & \first{0.000\%} \\
kroA100  & 4.017\%  & 28.828\% & 18.313\% & 0.728\% & \second{0.050\%} & \first{0.000\%}  & \first{0.00\%}  & \first{0.000\%}  & \first{0.00\%}  & \first{0.000\%}  & \first{0.000\%} \\
kroB100  & 5.142\%  & 34.686\% & 1.119\% & 0.147\% & \second{0.006\%} & \first{0.000\%}  & 0.65\%  & \first{0.000\%}  & \first{0.00\%}  & 0.742\%  & \first{0.000\%} \\
kroC100  & 0.972\%  & 35.506\% & \second{0.349\%} & 1.571\% & \first{0.000\%}  & \first{0.000\%}  & \first{0.00\%}  & \first{0.000\%}  & \first{0.00\%}  & 0.355\%  & \first{0.000\%} \\
kroD100  & 2.717\%  & 38.018\% & 0.866\% & \second{0.572\%} & \first{0.000\%}  & \first{0.000\%}  & \first{0.00\%}  & \first{0.000\%}  & \first{0.00\%}  & \first{0.000\%}  & \first{0.000\%} \\
kroE100  & 1.470\%  & 26.589\% & 1.832\% & \second{1.216\%} & \first{0.000\%}  & \first{0.000\%}  & \first{0.00\%}  & \first{0.000\%}  & \first{0.00\%}  & \first{0.000\%}  & \first{0.000\%} \\
rd100    & 3.407\%  & 50.432\% & \second{0.003\%} & 0.459\% & \first{0.000\%}  & \first{0.000\%}  & \first{0.00\%}  & \first{0.000\%}  & \first{0.00\%}  & 0.100\%  & \first{0.000\%} \\
eil101   & 2.994\%  & 26.701\% & 0.387\% & 0.201\% & \second{0.124\%} & \first{0.000\%}  & \first{0.00\%}  & \first{0.000\%}  & \first{0.00\%}  & \first{0.000\%}  & \first{0.000\%} \\
lin105   & 1.739\%  & 34.902\% & 1.867\% & 0.606\% & 0.441\% & 0.393\% & \first{0.00\%}  & \first{0.000\%}  & 0.54\%  & \first{0.000\%}  & \second{0.030\%} \\
pr107    & 3.933\%  & 80.564\% & 0.898\% & 0.439\% & 0.714\% & 0.155\% & 0.62\%  & \first{0.000\%}  & 0.08\%  & 0.100\%  & \second{0.040\%} \\
pr124    & 3.677\%  & 70.146\% & 10.322\% & 0.755\% & 0.997\% & 0.584\% & 0.08\%  & \first{0.000\%}  & 0.58\%  & 0.783\%  & \second{0.075\%} \\
bier127  & 5.908\%  & 45.561\% & 3.044\% & 1.948\% & 1.064\% & 0.718\% & 1.50\%  & \first{0.000\%}  & 0.66\%  & \first{0.000\%}  & \second{0.220\%} \\
ch130    & 3.182\%  & 39.090\% & 0.709\% & 3.519\% & 0.077\% & 0.077\% & \first{0.00\%}  & \first{0.000\%}  & \first{0.00\%}  & 0.597\%  & \second{0.025\%} \\
pr136    & 5.064\%  & 58.673\% & \first{0.000\%}  & 3.387\% & 0.182\% & \first{0.000\%}  & \second{0.01\%} & \first{0.000\%}  & \first{0.00\%}  & 0.146\%  & \first{0.000\%} \\
pr144    & 7.641\%  & 55.837\% & 1.526\% & 3.581\% & 1.816\% & \first{0.000\%}  & \second{0.39\%} & \first{0.000\%}  & \first{0.00\%}  & \first{0.000\%}  & \first{0.000\%} \\
ch150    & 4.584\%  & 49.743\% & 0.312\% & 2.113\% & 0.473\% & 0.324\% & \first{0.00\%}  & \first{0.000\%}  & \first{0.00\%}  & 0.893\%  & \second{0.030\%} \\
kroA150  & 3.784\%  & 45.411\% & 0.724\% & 2.984\% & 0.193\% & 0.193\% & \first{0.00\%}  & \first{0.000\%}  & \second{0.02\%} & 0.100\%  & 0.025\% \\
kroB150  & 2.437\%  & 56.745\% & 0.886\% & 3.258\% & 0.366\% & 0.021\% & 0.07\%  & 0.340\%  & \second{0.02\%} & 0.191\%  & \first{0.013\%} \\
pr152    & 7.494\%  & 33.925\% & \first{0.029\%}  & 3.119\% & 0.687\% & 0.687\% & 0.19\%  & 0.187\% & 0.68\%  & 0.191\%  & \second{0.180\%} \\
u159     & 7.551\%  & 38.338\% & \second{0.054\%} & 1.020\% & \first{0.000\%}  & \first{0.000\%}  & \first{0.00\%}  & \first{0.000\%}  & \first{0.00\%}  & \first{0.000\%}  & \first{0.000\%} \\
rat195   & 6.893\%  & 24.968\% & 0.743\% & 1.666\% & 0.887\% & \second{0.018\%} & 0.79\%  & 0.194\% & 0.11\%  & 0.419\%  & \first{0.012\%} \\
d198     & 373.020\% & 62.351\% & 0.522\% & 4.772\% & \first{0.000\%}  & \first{0.000\%}  & 0.86\%  & 0.751\% & \first{0.00\%}  & 0.488\%  & \second{0.350\%} \\
kroA200  & 7.106\%  & 40.885\% & 1.441\% & 2.029\% & 0.259\% & \first{0.000\%}  & 0.49\%  & 0.160\% & \first{0.00\%}  & 0.056\%  & \second{0.045\%} \\
kroB200  & 8.541\%  & 43.643\% & 2.064\% & 2.589\% & 0.171\% & 0.171\% & 2.50\%  & 0.098\% & \first{0.00\%}  & 0.851\%  & \second{0.085\%} \\
\midrule
\textbf{Mean} & 16.768\% & 40.025\% & 1.725\% & 1.529\% & 0.320\% & 0.133\% & 0.28\% & \second{0.060\%} & 0.09\% & 0.207\% & \first{0.039\%} \\
\bottomrule
\end{tabular}%
}
\end{table}

Table~\ref{tab:tsplib_results} reports optimality gaps across 29 TSPLIB instances at sizes 51--200. EDISCO leads on aggregate at $0.039\%$, ahead of GenSCO ($0.060\%$), NEXCO ($0.09\%$), and T2T ($0.133\%$). Older non-diffusion baselines (AM, GCN at autoregressive and supervised-construction settings, Learn2OPT and GNNGLS as 2-opt-based local-search heuristics) have mean gaps of $1.529$--$40.025\%$ and frame the difficulty of TSPLIB transfer for the older lineage. Each modern diffusion and expansion row is reported under its canonical decoding configuration: DIFUSCO, T2T, Fast-T2T, and EDISCO at sampling-based decoding with 2-opt, and GenSCO, NEXCO, and COExpander at greedy with 2-opt. EDISCO finds the optimum on every small instance (eil51 through eil101, plus pr136, pr144, u159), takes the strict best on kroB150 ($0.013\%$) and rat195 ($0.012\%$), and is the second-best on most remaining rows. The largest EDISCO gaps appear on the harder mid-range instances (d198 $0.350\%$, bier127 $0.220\%$, pr152 $0.180\%$), but every per-instance gap stays below $0.4\%$, giving the most uniform behavior in the table.

\subsection{Noise Schedule Design for TSP Diffusion}
\label{subsec:noise-schedule-experiment}

We conducted a comprehensive comparison of different noise schedule designs for the continuous-time categorical diffusion model applied to TSP. The noise schedule $\beta(t)$ controls the rate of information destruction during the forward diffusion process and significantly impacts model performance.

We evaluated three families of noise schedules:

\textbf{Linear Schedule:} Following the standard approach in diffusion models \citep{ho2020denoising}, we tested linear schedules with:
\begin{equation}
\beta(t) = \beta_{\min} + t(\beta_{\max} - \beta_{\min}).
\end{equation}

\textbf{Exponential Schedule:} Based on \citet{campbell2022continuous}, we evaluated exponential schedules:
\begin{equation}
\beta(t) = ab^t \log(b).
\end{equation}

\textbf{Cosine Schedule:} Following \citet{sun2022score}, we tested cosine schedules with improved numerical stability:
\begin{equation}
\beta(t) = \text{clip}\left(\frac{\pi}{4} \cdot \frac{\tan(\pi t/2)}{\sqrt{\cos(\pi t/2) + \epsilon}}, \beta_{\min}, \beta_{\max}\right).
\end{equation}

\begin{table}[!h]
\centering
\caption{Comparison of noise schedules for TSP diffusion on TSP-50. All models trained for 50 epochs with 50 diffusion steps. Bold indicates the lowest gap and inference time.}
\label{tab:schedule_comparison}
\tiny
\begin{tabular}{@{}llcccc@{}}
\toprule
Schedule & Configuration & \multicolumn{2}{c}{Optimality Gap (\%)$\downarrow$} & Conv. & Inference \\
\cmidrule(lr){3-4}
Type & & Best & Final & Epoch$\downarrow$ & Time (s)$\downarrow$ \\
\midrule
\multirow{3}{*}{Linear} 
  & Aggressive: $\beta \in [0.1, 2.0]$ & 2.88 & 3.03 & 50 & 1.06 \\
  & Baseline: $\beta \in [0.1, 1.5]$ & \textbf{2.29} & \textbf{2.63} & 45 & 1.06 \\
  & Conservative: $\beta \in [0.1, 1.0]$ & 2.74 & 3.51 & 50 & 1.16 \\
\midrule
\multirow{3}{*}{Exponential} 
  & Baseline: $a$=0.5, $b$=4.0 & 3.39 & 3.60 & 50 & \textbf{1.04} \\
  & Conservative: $a$=0.3, $b$=3.0 & 4.19 & 4.42 & 50 & 1.13 \\
  & Aggressive: $a$=0.8, $b$=5.0 & 2.60 & 3.78 & 50 & 1.05 \\
\midrule
\multirow{3}{*}{Cosine} 
  & Aggressive: $\beta \in [0.001, 10.0]$ & 4.83 & 5.37 & 45 & 1.05 \\
  & Baseline: $\beta \in [0.01, 5.0]$ & 3.45 & 4.51 & 40 & 1.07 \\
  & Conservative: $\beta \in [0.1, 3.0]$ & 3.25 & 4.09 & 40 & 1.06 \\
\bottomrule
\end{tabular}%
\end{table}

We trained each schedule variant for 50 epochs on TSP-50 using 10,000 randomly generated instances for fast verification. The same network architecture and hyperparameters were maintained across all experiments. Each schedule family was tested with three different parameterizations: baseline (standard parameters), conservative (slower noise injection with smaller $\beta$ values), and aggressive (faster noise injection with larger $\beta$ values).

Table~\ref{tab:schedule_comparison} presents the comprehensive results. Linear schedules demonstrated the best overall performance, with the baseline configuration ($\beta_{\min}=0.1, \beta_{\max}=1.5$) achieving the lowest validation gap of 2.29\%. The aggressive variant ($\beta_{\max}=2.0$) and conservative variant ($\beta_{\max}=1.0$) both underperformed at 2.88\% and 2.74\% respectively, supporting moderate noise injection in this TSP-50 experiment.

Exponential schedules showed high sensitivity to parameter selection, with best validation gaps ranging from 2.60\% to 4.19\% and final gaps ranging from 3.60\% to 4.42\% across configurations. Cosine schedules consistently underperformed with gaps ranging from 3.25\% to 4.83\%, indicating that their non-linear noise profile is less effective in this TSP-50 experiment.

The superiority of linear schedules in TSP contrasts with image generation, where cosine schedules often excel \citep{nichol2021improved}. We attribute this to the discrete nature of TSP adjacency matrices, which benefit from uniform, predictable noise injection rather than variable rates. These findings validate our choice of $\beta_{\min}=0.1, \beta_{\max}=1.5$ for all experiments, based on the tested TSP-50 configurations.

\subsection{Evaluation on Architectural Hyperparameters}
\label{subsec:archi_hyperparameters}
We conduct systematic evaluations of critical architectural hyperparameters to justify our design choices.

\paragraph{Step Size $\alpha$ for Coordinate Updates}

The step size $\alpha$ in Equation~\ref{eq:coord_update} controls the magnitude of coordinate updates during message passing. We evaluate different values on TSP-50:

\begin{table}[!h]
\centering
\caption{Effect of step size $\alpha$ on model performance and training stability (TSP-50). $^*$Standard deviation of coordinates after 8 layers (optimal range: 0.3-0.4).}
\label{tab:ablation_alpha}
\begin{tabular}{@{}lccccc@{}}
\toprule
Step Size $\alpha$ & 0.01 & 0.05 & \textbf{0.1} & 0.2 & 0.5 \\
\midrule
Optimality Gap (\%)$\downarrow$ & 0.08 & 0.03 & \textbf{0.01} & 0.15 & Diverged \\
Coordinate Std$^*$ & 0.42 & 0.38 & \textbf{0.35} & 0.18 & 0.02 \\
Training Stable & \checkmark & \checkmark & \checkmark & Unstable & Collapsed \\
Convergence Epoch$\downarrow$ & 42 & 38 & \textbf{35} & 48 & - \\
\bottomrule
\end{tabular}
\end{table}

As shown in Table~\ref{tab:ablation_alpha}, smaller $\alpha$ values (0.01, 0.05) maintain stability but converge more slowly and achieve suboptimal performance. Larger values ($\alpha \geq 0.2$) cause instability and coordinate contraction, with the coordinate standard deviation falling to 0.18 at $\alpha=0.2$ and 0.02 at $\alpha=0.5$, where training collapses.

\paragraph{Temperature Parameter $\tau$ for Weight Scaling}

The temperature parameter $\tau$ in $w_{ij}=\tanh(\mathrm{MLP}_c(\mathbf{m}_{ij})/\tau)$ scales the MLP outputs before applying tanh, preventing gradient saturation:

\begin{table}[!h]
\centering
\caption{Effect of temperature $\tau$ on gradient flow and performance (TSP-50). $^{\dagger}$Average gradient norm in coordinate MLP during first 10 epochs. $^{\ddagger}$Percentage of tanh outputs with $|w_{ij}| > 0.95$.}
\label{tab:ablation_tau}
\begin{tabular}{@{}lccccc@{}}
\toprule
Temperature $\tau$ & 1 & 5 & \textbf{10} & 20 & 50 \\
\midrule
Optimality Gap (\%)$\downarrow$ & 0.12 & 0.04 & \textbf{0.01} & 0.02 & 0.05 \\
Avg. Gradient Norm$^{\dagger}$ & 0.003 & 0.018 & \textbf{0.042} & 0.038 & 0.031 \\
Tanh Saturation Rate$^{\ddagger}$ & 68\% & 24\% & \textbf{8\%} & 12\% & 18\% \\
Convergence Epoch$\downarrow$ & 52 & 40 & \textbf{35} & 36 & 39 \\
\bottomrule
\end{tabular}
\end{table}

Table~\ref{tab:ablation_tau} shows that $\tau=10$ achieves the optimal balance. Lower values ($\tau \leq 5$) cause excessive saturation, leading to vanishing gradients and slower convergence. Higher values ($\tau \geq 20$) reduce the non-linearity's effectiveness, diminishing the model's expressiveness.

\paragraph{Joint Impact Analysis}

We evaluate the joint effect of $\alpha$ and $\tau$ on TSP-100 performance:

\begin{table}[!h]
\centering
\caption{Joint ablation of $\alpha$ and $\tau$ on TSP-100 (optimality gap \%$\downarrow$).}
\label{tab:alpha_joint_ablation}
\begin{tabular}{@{}l|ccccc@{}}
\toprule
\diagbox{$\alpha$}{$\tau$} & 1 & 5 & 10 & 20 & 50 \\
\midrule
0.01 & 0.21 & 0.15 & 0.08 & 0.09 & 0.12 \\
0.05 & 0.18 & 0.09 & 0.05 & 0.06 & 0.08 \\
0.1 & 0.15 & 0.06 & \textbf{0.04} & 0.05 & 0.07 \\
0.2 & Unstable & 0.18 & 0.15 & 0.16 & 0.19 \\
0.5 & Collapsed & Collapsed & Diverged & Diverged & Diverged \\
\bottomrule
\end{tabular}
\end{table}

The joint ablation confirms that $\alpha=0.1$ and $\tau=10$ provide the best tested configuration, with larger step sizes leading to instability or divergence.

\paragraph{Discussion on Hyperparameter Robustness}

While our ablation studies reveal sensitivity to $\alpha$ and $\tau$ during initial design exploration—particularly coordinate contraction and instability when $\alpha \geq 0.2$—we emphasize three important aspects of these architectural stabilizers: (1) \textbf{Optimal settings are robust}: Once identified, $\alpha=0.1$ and $\tau=10$ work consistently across all problem sizes and types (TSP-50/100/500/1000, CVRP-50/100/200/500, ESTP-10/20/50) without per-problem retuning, as demonstrated throughout our experiments. (2) \textbf{Comparable to standard practice}: This sensitivity is analogous to other architectural hyperparameters in deep learning (e.g., learning rates, attention heads, batch normalization momentum), which require careful tuning during model development but remain stable once configured. (3) \textbf{Strong generalization validates stability}: Our cross-distribution evaluation (Table~\ref{tab:ood_results} in the main text) demonstrates that these settings generalize well to out-of-distribution instances (Cluster, Explosion, Implosion), confirming they are not narrowly tuned to specific data distributions.

The thorough sensitivity analysis presented here provides practitioners with clear guidance for implementation and demonstrates our commitment to transparency in reporting both the strengths and limitations of our architectural design.

\subsection{Model Efficiency}
\label{subsec:model_efficiency} 

Although the results in the main text are from the full-scale EDISCO model, it is interesting to see EDISCO's performance under reduced model sizes. Table~\ref{tab:model_efficiency} presents a comprehensive comparison of model efficiency across different architectures and scales, demonstrating how EDISCO's equivariant design enables strong performance even with reduced model capacity and training data.

\begin{table}[!h]
\centering
\caption{Model efficiency and performance comparison across TSP scales. Training instances shown in thousands (K) with corresponding optimality gaps (\%). All EDISCO entries use 50-step PNDM with greedy decoding (no NEE, no 2-opt) so capacity differences across architectures are isolated from sampler and decoder effects; the main-table SL+NEE configuration of Table~\ref{tab:tsp_main} therefore reaches lower gaps than the EDISCO-Full row here. EDISCO-Full uses 12 EGNN layers with 256 hidden dimension (5.5M parameters), EDISCO-Medium uses 12 layers with 128 hidden dimension (1.9M parameters), and EDISCO-Small uses 8 layers with 128 hidden dimension (1.3M parameters). $^{\dagger}$12-layer GNN with width 256. $^{*}$Same architecture as DIFUSCO.}
\label{tab:model_efficiency}
\small
\begin{tabular}{@{}lcccccc@{}}
\toprule
\multirow{2}{*}{Method} & Model & \multicolumn{5}{c}{Training Instances (K) / Optimality Gap (\%)$\downarrow$} \\
\cmidrule(lr){3-7}
& Parameters & TSP-50 & TSP-100 & TSP-500 & TSP-1000 & TSP-10000 \\
\midrule
DIFUSCO$^{\dagger}$ & 5.3M & 1500 / 0.48 & 1500 / 1.01 & 128 / 9.41 & 64 / 11.24 & 6.4 / 8.95 \\
DISCO$^{\dagger}$ & 5.3M & 1500 / 0.16 & 1500 / 0.58 & - / - & - / - & 6.4 / 2.90 \\
T2T$^{*}$ & 5.3M & 1500 / 0.04 & 1500 / 0.18 & 128 / 5.09 & 64 / 8.87 & 6.4 / 2.92 \\
\midrule
EDISCO-Full & 5.5M & 500 / \textbf{0.01} & 500 / \textbf{0.04} & 60 / \textbf{1.95} & 30 / \textbf{2.85} & 3 / \textbf{1.98} \\
EDISCO-Medium & 1.9M & 500 / 0.03 & 500 / 0.08 & 60 / 2.18 & 30 / \textbf{2.85} & 3 / 2.43 \\
EDISCO-Small & 1.3M & \textbf{300} / 0.08 & \textbf{300} / 0.7 & \textbf{40} / 4.18 & \textbf{20} / 5.21 & \textbf{2} / 3.18 \\
\bottomrule
\end{tabular}%
\end{table}

EDISCO-Full (5.5M parameters) uses one-third as much training data as baselines on small instances (500K vs 1.5M) and about half as much on large instances. It achieves 0.01\% gap on TSP-50 compared to DIFUSCO's 0.48\%, and 1.95\% on TSP-500 versus DIFUSCO's 9.41\%. On TSP-10000, EDISCO-Full achieves 1.98\% gap with 3K training instances, outperforming DIFUSCO's 8.95\% gap with 6.4K instances.

EDISCO-Medium (1.9M parameters), with about one-third of the parameters of baseline models, achieves 0.03\% gap on TSP-50 and 2.18\% on TSP-500. It matches EDISCO-Full's performance on TSP-1000 at 2.85\% gap. Compared to T2T, which achieves 0.04\% on TSP-50 with 5.3M parameters, EDISCO-Medium achieves comparable performance (0.03\%) with 1.9M parameters and the same 500K training instances used by EDISCO-Full.

EDISCO-Small (1.3M parameters) uses about 24\% as many parameters as baselines and requires the least training data: 300K for TSP-50/100, 40K for TSP-500, and 2K for TSP-10000. It achieves 0.08\% gap on TSP-50 and 0.7\% on TSP-100. On TSP-500, with 40K training instances, it achieves 4.18\% gap, compared to DIFUSCO's 9.41\% with 128K instances. On TSP-10000, EDISCO-Small achieves 3.18\% gap, which outperforms DIFUSCO (8.95\%) and is comparable to DISCO (2.90\%).

The results for the amount of training data represent the minimum numbers required to reach the reported EDISCO gaps. After this, even with increased training data, there is no noticeable improvement in the optimality gaps.

\subsection{Training Resource Requirements}
\label{subsec:training_resources}

Table~\ref{tab:training_resources} presents comprehensive training resource requirements for EDISCO-Full, demonstrating its efficiency on commodity hardware. EDISCO trains on a single NVIDIA RTX A6000 GPU, using graph sparsification and fewer training instances. The combination of E(2)-equivariance and continuous-time diffusion enables 3× better sample efficiency (500K vs 1.5M training instances for TSP-50/100), reducing the required number of training instances. Measured training time and peak memory are reported below.

\begin{table}[!h]
\centering
\caption{Training resource requirements for EDISCO-Full on a single NVIDIA RTX A6000 GPU (48\,GB). Peak memory usage includes model parameters, optimizer states, and batch processing. EDISCO achieves state-of-the-art performance using 33--50\% of the training data of baseline methods while training on a single GPU.}
\label{tab:training_resources}
\adjustbox{max width=\textwidth}{%
\begin{tabular}{@{}lccccc@{}}
\toprule
\textbf{Problem Scale} & \textbf{Dataset Size} & \textbf{Batch Size} & \textbf{Training Epochs} & \textbf{Training Time}$\downarrow$ & \textbf{Peak Memory}$\downarrow$ \\
\midrule
TSP-50  & 500K & 64 & 100 & 18h 45m  & 19.2 GB \\
TSP-100 & 500K & 32 & 100 & 68h 30m  & 26.8 GB \\
TSP-500 & 60K  & 16 & 50  & 31h 15m  & 23.5 GB \\
TSP-1000 & 30K & 8  & 50  & 61h 20m  & 39.2 GB \\
TSP-10000 & 3K & 4  & 50  & 17h 40m  & 44.8 GB \\
\bottomrule
\end{tabular}%
}
\end{table}

Graph sparsification ($k$NN with $k=50$ for TSP-500 and $k=100$ for TSP-1000/10000) reduces memory complexity from $O(n^2)$ to $O(kn)$, enabling practical training on a single GPU even for TSP-10000. TSP-500/1000/10000 use curriculum learning initialized from smaller problem checkpoints, further improving sample efficiency. Peak memory remains within the A6000's 48\,GB capacity across all reported TSP scales.

\subsection{Visualization of EDISCO's E(2)-Equivariance}
\label{subsec:visualization}
Figure~\ref{fig:edisco_original} illustrates the denoising process of EDISCO on a standard TSP-100 instance. To illustrate the E(2)-equivariant behavior of the network qualitatively, we present the same denoising process on rotated coordinate sets in Figures~\ref{fig:edisco_45} and~\ref{fig:edisco_90}. Only the greedy decoder is used without any other post-processing techniques. The visualization shows progression from pure noise ($t=1.0$) to clean tours ($t=0.0$) across five independent sampling rounds. The rotated coordinate sets are renormalized to the unit square for display, so the per-figure tour lengths reflect each panel's display scale rather than absolute coordinate length; the structural correspondence of the denoising trajectories is what illustrates the rotation-invariant representations established by Corollary~\ref{cor:orbit_sharing}, with the formal length-invariance argument in Appendix~\ref{subsec:equivarince}.

\begin{figure}[!h]
\centering
\includegraphics[width=\textwidth]{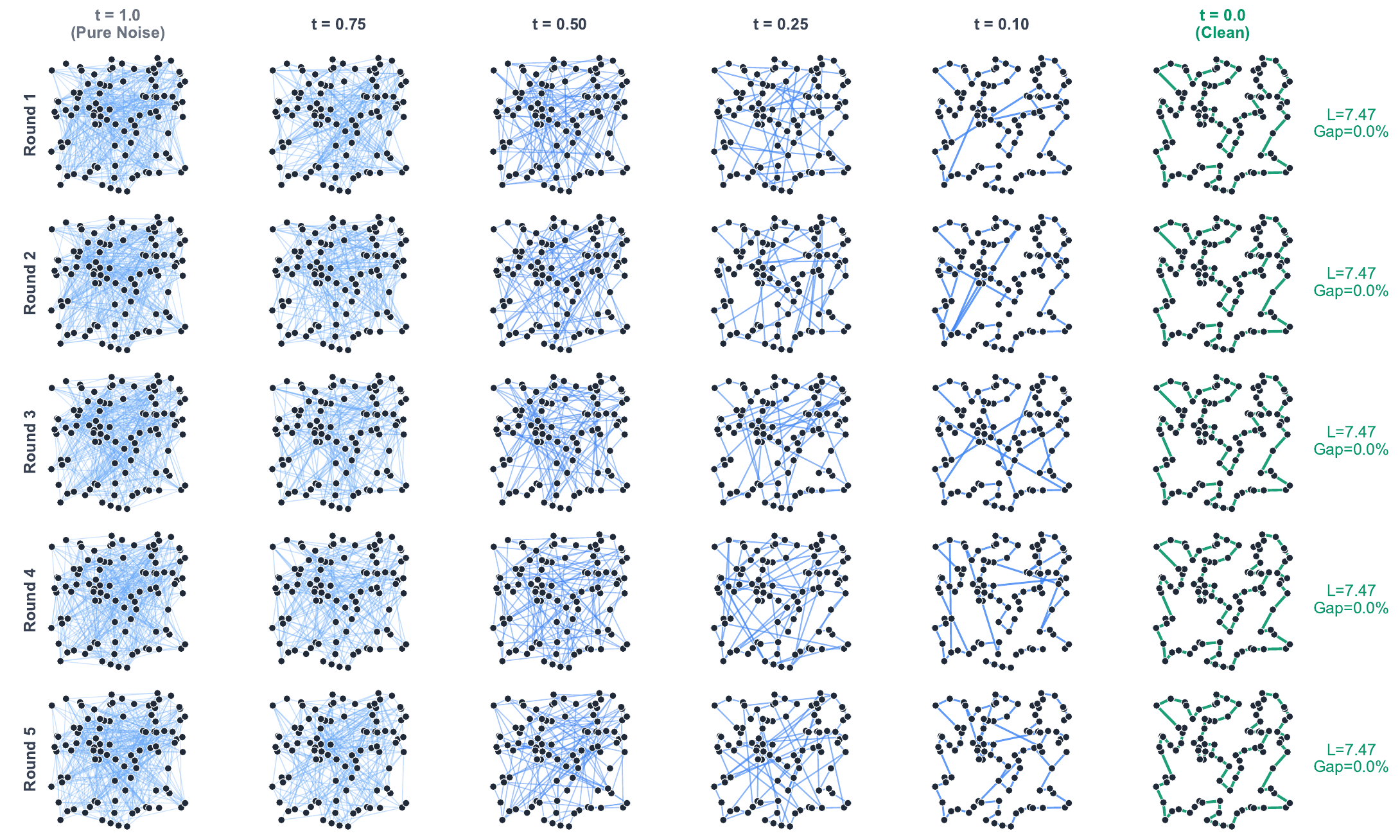}
\caption{Five rounds of the denoising process of EDISCO on the original TSP instance.}
\label{fig:edisco_original}
\end{figure}

\begin{figure}[!h]
\centering
\includegraphics[width=\textwidth]{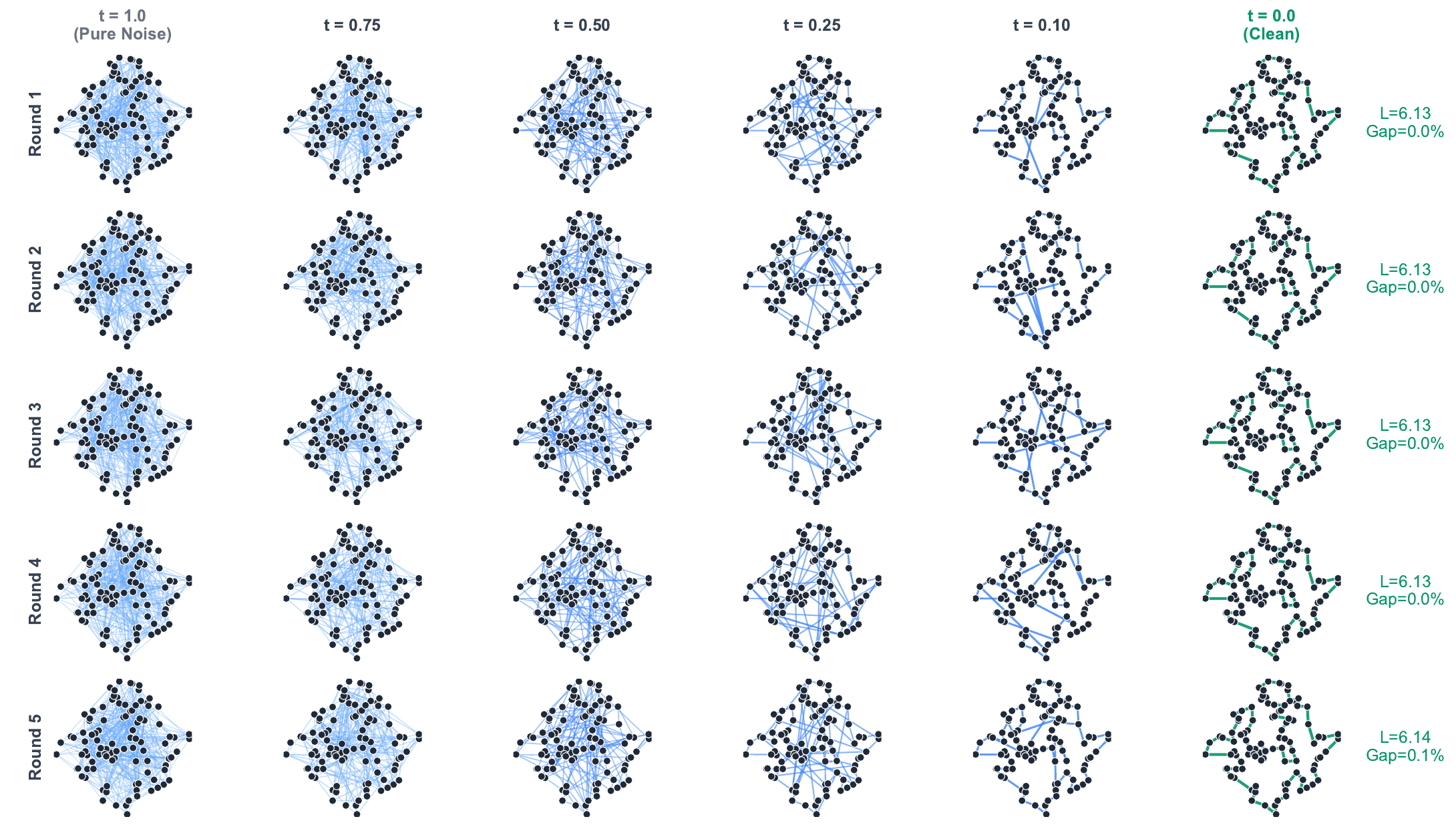}
\caption{Five rounds of the denoising process of EDISCO on the $45^\circ$ rotated instance.}
\label{fig:edisco_45}
\end{figure}

\begin{figure}[!h]
\centering
\includegraphics[width=\textwidth]{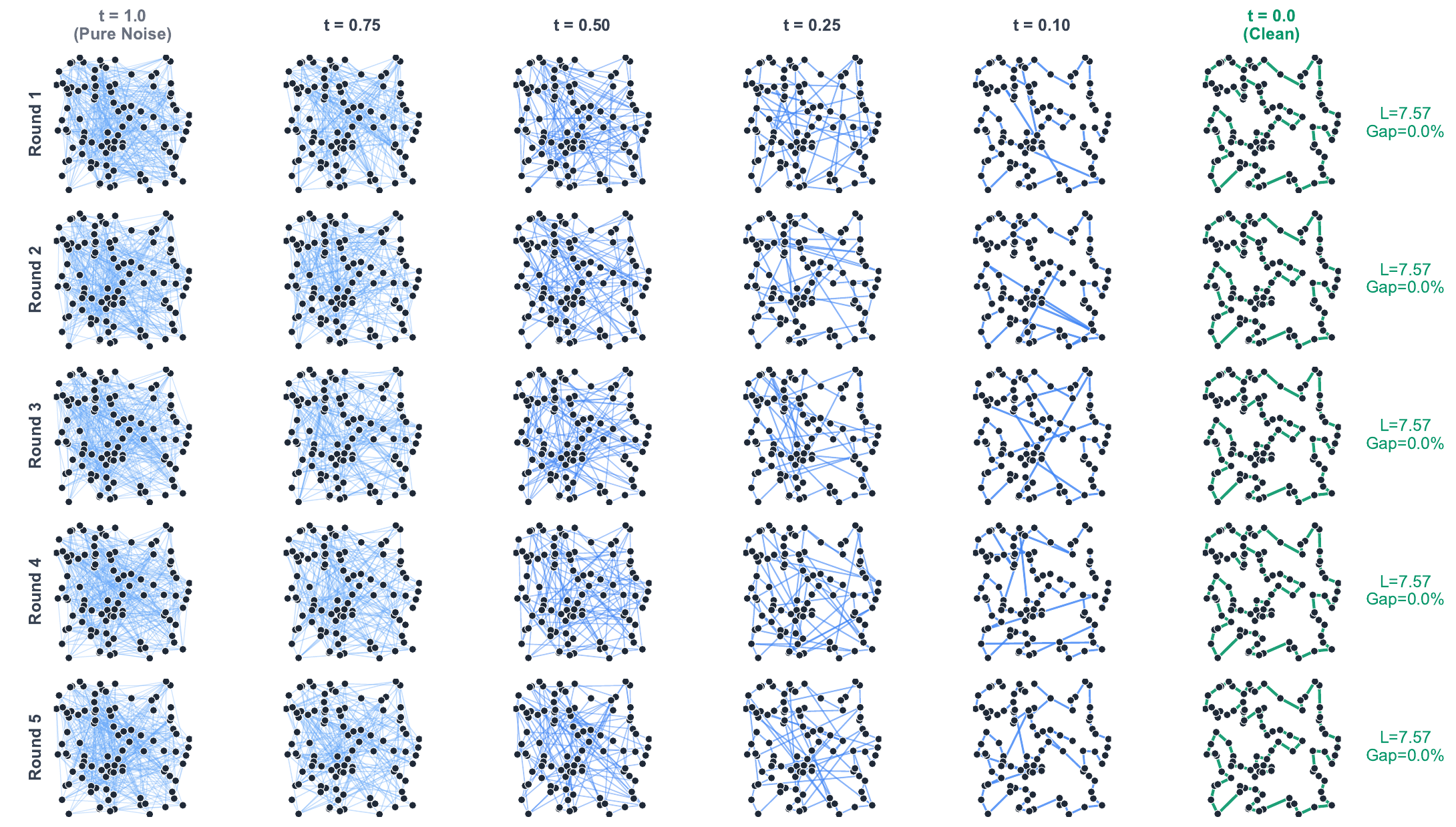}
\caption{Five rounds of the denoising process of EDISCO on the $90^\circ$ rotated instance.}
\label{fig:edisco_90}
\end{figure}

\subsection{Compute and Runtime Conventions for Recent Competitors}
\label{subsec:recent_competitors_pareto}

This appendix documents the compute and runtime conventions used for the recent-competitor entries (StruDiCO, GenSCO, NEXCO, COExpander) in Table~\ref{tab:tsp_main} of Section~\ref{subsec:tsp_main}. The compute axis differs by method: number of denoising steps for diffusion methods, number of search-driven generation cycles for GenSCO~\citep{li2025gensco}, and iteration depth of the adaptive expansion for COExpander~\citep{macoexpander}. EDISCO is reported at both PNDM-50 and DEIS-2-5 configurations. The Table~\ref{tab:tsp_main} snapshot uses the minimum budget at which each method achieves its reported gap.

\paragraph{Compute budget conventions.} All wall-clock times are per-instance, measured at batch size 1 in single-thread mode on the hardware described in Appendix~\ref{subsec:hardware}, following the protocol of~\citet{wang2025strudico,wang2026nexco}. Each method is run at its own canonical decoding configuration (denoising step counts for diffusion methods, search cycles for GenSCO, expansion iterations for COExpander), with timings following the stated measurement conventions and computational budgets remaining method-specific.

\end{document}